\documentclass[conference]{IEEEtran}
\IEEEoverridecommandlockouts

\usepackage{cite}
\usepackage{graphicx,subcaption,amsmath,amssymb,amsthm,hyperref,amsfonts, array, soul}
\graphicspath{ {./images/} }
\usepackage{dsfont} 
\usepackage{amsmath,amssymb,amsfonts}
\usepackage{algorithm}
\usepackage[noend]{algorithmic}
\usepackage{graphicx}
\usepackage{textcomp}
\usepackage{xcolor}
\def\BibTeX{{\rm B\kern-.05em{\sc i\kern-.025em b}\kern-.08em
    T\kern-.1667em\lower.7ex\hbox{E}\kern-.125emX}}

\allowdisplaybreaks
\newcommand{\beqano}{\begin{eqnarray*}}
\newcommand{\eeqano}{\end{eqnarray*}}
\newcommand{\beqa}{\begin{eqnarray}}
\newcommand{\eeqa}{\end{eqnarray}}

\newcommand{\bi}{\begin{itemize}}
\newcommand{\ei}{\end{itemize}}

\newcommand{\be}{\begin{enumerate}}
\newcommand{\ee}{\end{enumerate}}

\newcommand{\del}[1]{\Delta_{#1}}
\newcommand{\tdel}[1]{\tilde{\Delta}_{#1}}
\newcommand{\ucb}{\text{UCB}}
\newcommand{\lcb}{\text{LCB}}
\newcommand{\sync}{\text{sync}}
\newcommand{\ind}{\text{ind}}
\newcommand{\dir}{\text{dir}}
\newcommand{\free}{\text{free}}

\newcommand{\cN}{\mathcal{N}}
\newcommand{\cC}{\mathcal{C}}
\newcommand{\cP}{\mathcal{P}}

\newcommand{\bE}{\mathbb{E}}
\newcommand{\bP}{\mathbb{P}}
\newcommand{\cU}{\mathcal{U}}
\newcommand{\cG}{\mathcal{G}}
\newcommand{\cK}{\mathcal{K}}

\newcommand{\cD}{\mathcal{D}}
\newcommand{\cB}{\mathcal{B}}
\newcommand{\cM}{\mathcal{M}}
\newcommand{\cF}{\mathcal{F}}

\newcommand{\vz}{\mathbf{z}}

\newcommand{\E}{\mathbb{E}}

\newtheorem{assumption}{Assumption}
\newtheorem{theorem}{Theorem}
\newtheorem{lemma}{Lemma}

\begin{document}

\title{An LP-based Sampling Policy for Multi-Armed Bandits with Side-Observations and Stochastic Availability \thanks{This work has been supported in part by the Army Research Laboratory and was accomplished under Cooperative Agreement Number W911NF-23-2-0225, the ARO Grant W911NF-24-1-0103, and by the U.S. National Science Foundation under the grants: NSF AI Institute (AI-EDGE) 2112471, CNS-NeTS-2106679, CNS2312836, CNS-2225561, and CNS-2239677. Additionally, this research was supported by the Office of Naval Research under grant N00014-24-1-2729. The views and conclusions contained in this document are those of the authors and should not be interpreted as representing the official policies, either expressed or implied, of the Army Research Laboratory or the U.S. Government. The U.S. Government is authorized to reproduce and distribute reprints for Government purposes, notwithstanding any copyright notation herein.}}
\author{\IEEEauthorblockN{Ashutosh Soni\IEEEauthorrefmark{1}, Peizhong Ju\IEEEauthorrefmark{2}, Atilla Eryilmaz\IEEEauthorrefmark{1} and Ness B. Shroff\IEEEauthorrefmark{1}\IEEEauthorrefmark{3}}
\IEEEauthorblockA{\IEEEauthorrefmark{1}Department of Electrical and Computer Engineering, The Ohio State University, Columbus, Ohio, USA}
\IEEEauthorblockA{\IEEEauthorrefmark{2}Department of Computer Science, University of Kentucky, Lexington, Kentucky, USA}
\IEEEauthorblockA{\IEEEauthorrefmark{3} Department of Computer Science and Engineering, The Ohio State University, Columbus, Ohio, USA}
\IEEEauthorblockA{Email: \{soni.117, eryilmaz.2, shroff.11\}@osu.edu, peizhong.ju@uky.edu}
}


\maketitle

\begin{abstract}
We study the stochastic multi-armed bandit (MAB) problem where an underlying network structure enables side-observations across related actions.
We use a bipartite graph to link actions to a set of unknowns, such that selecting an action reveals observations for all the unknowns it is connected to.
While previous works rely on the assumption that all actions are permanently accessible, we investigate the more practical setting of stochastic availability, where the set of feasible actions (the ``activation set") varies dynamically in each round.
This framework models real-world systems with both structural dependencies and volatility, such as social networks where users provide side-information about their peers' preferences, yet are not always online to be queried.
To address this challenge, we propose UCB-LP-A, a novel policy that leverages a Linear Programming (LP) approach to optimize exploration-exploitation trade-offs under stochastic availability.
Unlike standard network bandit algorithms that assume constant access, UCB-LP-A computes an optimal sampling distribution over the realizable activation sets, ensuring that the necessary observations are gathered using only the currently active arms.
We derive a theoretical upper bound on the regret of our policy, characterizing the impact of both the network structure and the activation probabilities. 
Finally, we demonstrate through numerical simulations that UCB-LP-A significantly outperforms existing heuristics that ignore either the side-information or the availability constraints.
\end{abstract}


\section{Introduction}
The Multi-Armed Bandit (MAB) framework serves as a fundamental model for sequential decision-making under uncertainty \cite{lai1985asymptotically}
In the classical setting, a decision-maker sequentially selects an action and observes a reward drawn from an unknown distribution to minimize regret, which is defined as the difference between the total reward obtained from the action with the highest average reward and the given policy's total reward.

While standard MAB formulations assume actions are independent, many real-world systems exhibit an underlying structure where they can be correlated.
This means that choosing an action not only generates a reward for itself, but also reveals some useful \textit{side-information} for a subset of the remaining actions, significantly accelerating learning and reducing regret.
This setup has been studied in \cite{buccapatnam2018reward}, and it forms a basis for our proposed policy design and analysis.
We can model this relationship between different actions using a bipartite graph between the set of \textit{actions} and a common set of unknowns called \textit{base-arms} (see Figure \ref{fig:side_info_model_2}).
The reward from each action is a known function of a subset of the base-arms, and choosing an action reveals observations from all base-arms connected to this action.

Such a side-information structure between actions becomes available in a variety of applications. For example, consider the problem of \textit{routing} in communication networks, where packets are to be sent over a path (set of links) from source to destination.
The total delay is the sum of unknown individual link delays. 
Traversing one path reveals observations for delays on each of the constituent links and hence provides partial information about other paths that share those links.
Advertising in online social networks through promotional offers can be another example.
Assume that a social network platform targets user X with a promotion while also advertising the same to his friends, saying - `user X was just offered a promotion, would be interested in the same too?'
The responses of the friends would be independent of the fact that user X accepted or rejected the offer.

A critical limitation in existing literature is the assumption of constant action availability, i.e., every action is available at every time step. 
This assumption rarely holds in practice; for instance, social network users are not perpetually online, and wireless links are subject to intermittent outages. 
In this work, we extend the standard framework to model stochastic availability, where action unavailability is not purely random but statistically correlated. 
In social networks, for example, user activity is often group-dependent rather than independent. 
Similarly, in network routing, link failures are frequently spatially correlated—such as when localized interference or jamming attacks disable multiple related links simultaneously.

To model this, we introduce the concept of discrete \textit{activation sets}. 
Rather than assuming independent availability probabilities for each action, we assume that at any given moment, the set of available actions is drawn from a known collection of subsets (activation sets), each with a distinct probability of occurrence. 
This structure captures the correlated nature of availability, representing, for instance, specific configurations of active users in a social network. 
Our objective is to design a policy that identifies the optimal action within each activation set while leveraging the underlying graph-based side-information.

The introduction of activation sets fundamentally alters the exploration-exploitation trade-off since samples from actions available at some time can influence learning about actions available at a different time. 
Standard side-information algorithms (like UCB-LP from \cite{buccapatnam2018reward}, and UCB-MaxN from \cite{caron2012leveraging}) become infeasible if the solution to their optimization problem requires choosing an action that is currently inactive. 
Conversely, standard bandit algorithms ignore the graph structure, leading to inefficient sampling.

In this work, we aim to provide a policy that can exploit the side information structure in the presence of activation sets in a general stochastic MAB problem.
Our main contributions are as follows:
\begin{itemize}
    \item We formulate the stochastic MAB problem with side-observations and activation sets, modeling availability as a probabilistic process over subsets of arms.
    \item We develop the UCB-LP-A policy, which solves a Linear Program (LP) that explicitly accounts for the side information structure along with the activation sets and their probabilities, ensuring a more efficient method of sampling.
    \item We provide a theoretical analysis of UCB-LP-A, deriving a regret upper bound that characterizes the dependency on the network structure and the activation probabilities.
    \item We demonstrate through simulations on both synthetic and real-world network topologies that UCB-LP-A outperforms existing baselines that fail to account for either the side-information or the availability constraints.
\end{itemize}

The model considered in this work is an important first step in the direction of more general models that can potentially learn the correlations between action availability on the go instead of assuming knowledge.

\section{Related Works}
\textbf{Side-Observation:}
This line of research explores models where choosing a single action reveals feedback for multiple neighboring arms, typically formalized through a graph-based side-observation structure.
The linear scaling of regret with the number of suboptimal arms renders traditional bandit policies ineffective for large action spaces typical of content recommendation and advertising.
To address this limitation, researchers have developed richer models that leverage additional information shared across reward distributions. Prominent examples include dependent bandits in \cite{pandey2007multi}, X-armed bandits in \cite{bubeck2011x}, linear bandits in\cite{rusmevichientong2010linearly}, contextual bandits with side observation in \cite{li2010contextual}, and combinatorial bandits in \cite{chen2013combinatorial}.
The works in \cite{mannor2011bandits}, \cite{caron2012leveraging} (proposed UCB-N and UCB-MaxN policies), and \cite{buccapatnam2014stochastic} handled the large number of actions by assuming that choosing an action reveals observations from a larger set of actions.
The work in \cite{buccapatnam2018reward} extends the setting used in \cite{caron2012leveraging} and \cite{buccapatnam2014stochastic} to a more general graph feedback structure between a set of actions and a set of common unknowns.
This is the setting used in this paper, and the proposed policy builds on the ideas from this work.
The work in \cite{cohen2016online} studied the MAB problem with graph based feedback structure similar to \cite{mannor2011bandits} and \cite{buccapatnam2014stochastic} but assumed that the structure is never fully revealed.
In many real-world cases, like routing problems in communication networks, the graph structure is known or is learnt apriori.
In \cite{chen2013combinatorial}, the authors considered that a subset of base-actions form super actions, and in each round, choosing a super action reveals the outcome of its constituent actions.
Their proposed policy does not utilize the underlying network structure between base actions and super actions.

\textbf{Stochastic Availability:}
The assumption of constant arm availability is frequently violated in real-world scenarios. 
The problem of time-varying availability is primarily studied under the ``Sleeping Bandit'' framework initiated in \cite{kleinberg2010regret}, who established regret bounds against the best available action, and \cite{kanade2009sleeping} extends this idea for the adversarial setting.
More recently, \cite{li2019combinatorial} studied the combinatorial version of sleeping bandits with fairness constraints.
Other extensions include ``Volatile Bandits'' from \cite{bnaya2013volatile} and ``Mortal Bandits'' from \cite{chakrabarti2008mortal}, which model stochastic availability in dynamic environments.
However, none of these approaches explicitly leverage the \textit{side observation} structure central to our work. 

Our work bridges this gap between the two sets of work by integrating the graph-based side observation model of \cite{buccapatnam2018reward} with the stochastic activation sets.




\section{Problem Formulation}
In this section, we formally define the general bandit problem with activation sets in the presence of side-information across actions.
Let $\cN = \{1,\dots,N\}$ be a set of \textit{base-arms} with unknown distributions, and let $\cK = \{1,\dots,K\}$ denote the set of actions. 
At each time $t$, an \textit{activation set} $\cK_t$ is sampled from a collection $\{\cK_1, \dots, \cK_A\}$ according to a fixed distribution $\mathbf{p}$, such that $\mathbb{P}(\cK_t = \cK_a) = p_a$. 
The decision maker must select an action $k \in \cK_t$ (e.g., choosing a currently active user in a social network).


Let $X_i(t)$ denote the random reward of base-arm $i$ at time $t$, assumed independent and identically distributed (i.i.d.) over time and independent across base-arms. 
Choosing action $j$ reveals outcomes for the set of base-arms $\cC_j \subseteq \cN$. 
Conversely, let $\cP_i = \{j \in \cK : i \in \cC_j\}$ denote the set of actions that observe base-arm $i$. 
While choosing action $j$ yields observations for the entire set $\cC_j$, the instantaneous reward is computed via a known function $f_j(\cdot)$ based only on a subset $\cF_j \subseteq \cC_j$. 
This distinction $\cF_j \subseteq \cC_j$ captures the 
idea that some base-arms are observed by an action but do not contribute to its reward.
Let $\textbf{X}_j(t) = [X_i(t)]_{i\in\cF_j}$ denote the vector of relevant outcomes. The reward for action $j$ at time $t$ is then determined by $f_j(\textbf{X}_j(t))$.


We assume that the reward is bounded in $[0,1]$ for each action.
Note that we only assume that the reward function $f_j(.)$ is bounded, and the specific form of $f_j(.)$ and $\cF_j$ is determined by the decision maker or the specific problem. 
Let $\mu_j$ represent the mean of reward on playing action $j$.
Let $i^*_a$ and $\mu^*_a$ represent the optimal action and the optimal reward for activation set $\cK_a$.
Denote $\cU_a=\cK_a\setminus i^*_a$ as the set of suboptimal actions for set $\cK_a$.
Different $\cK_a$ can have the same optimal action, as an action $j$ can belong to many activation sets.

\subsection{Side-information model}

The structural relationship between actions $\cK$ and base-arms $\cN$ is modeled by a bipartite graph $G=(\cK,\cN,E)$ with adjacency matrix $E=[e_{i,j}]$, where $e_{i,j}=1$ if $i \in \cC_j$ and $0$ otherwise. 
An edge $(j,i)$ implies that choosing action $j$ yields a realization of base-arm $i$. 
Intuitively, the graph $G$ captures the side-observation capabilities, while the collection $\{\cF_j\}_{j\in\cK}$ dictates the reward structure. 
We assume $\bigcup_{j\in\cK}\cF_j=\cN$, ensuring no redundant base-arms exist in the system.

\begin{figure}[h]
    \centering
    \includegraphics[scale=0.27]{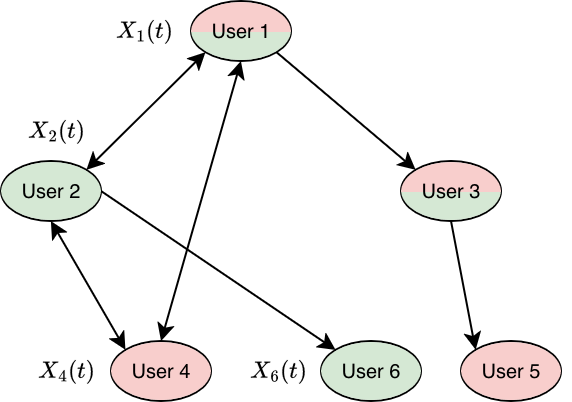}
    \caption{Example of a social network with 2 activation sets and side information.}
    \label{fig:side_info_model_1}
\end{figure}


Figure \ref{fig:side_info_model_1} visualizes the framework in a social network context where the set of actions (users) is also the set of base-arms ($\cK=\cN$). 
Two activation sets, $\cK_1$ (red) and $\cK_2$ (green), represent distinct groups of currently online users. 
We can observe that users 1 and 3 participate in both sets.
The network structure facilitates side-observations; for instance, if $\cK_2$ is active, selecting user 2 reveals its own response $X_2(t)$ along with the responses of its neighbors $\{X_1(t), X_4(t), X_6(t)\}$.
\begin{figure}[h]
    \centering
    \includegraphics[scale=0.27]{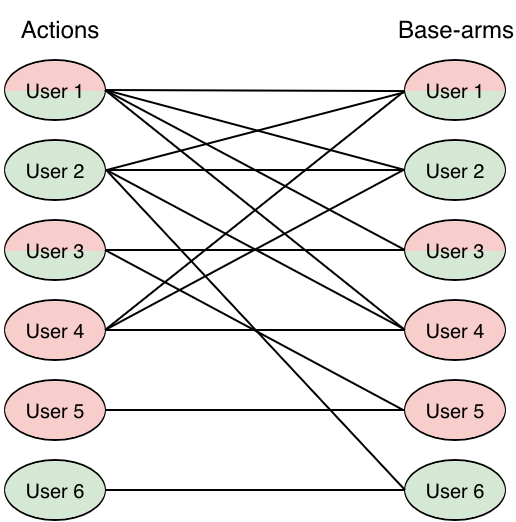}
    \caption{Bipartite graph for the example of targeting users in an online social network.}
    \label{fig:side_info_model_2}
\end{figure}

Figure \ref{fig:side_info_model_2} presents the bipartite representation of the previous example, distinguishing activation sets $\cK_1$ and $\cK_2$ by color. 
The graph explicitly maps observation dependencies; for instance, selecting action $2$ yields $\cC_2=\{1,2,4,6\}$. The actual reward is then determined by the specification of $\cF_2$. If $\cF_2=\{1,2,4,6\}$ and $f_2(\cdot)$ is a summation, the reward aggregates feedback from all neighbors. Conversely, if $\cF_2=\{2\}$, the reward is derived solely from user 2, treating observations from $\{1,4,6\}$ as pure side-information.



The framework accommodates general bounded reward functions and extends beyond simple social graphs. 
The bipartite structure can model higher-order dependencies, such as observing responses from a ``friend of a friend'' (two-hop neighborhood). 
Additionally, edges may represent latent similarity rather than explicit friendship, linking users with comparable preference profiles, a common abstraction in recommendation systems like Yelp.

\subsection{Objective}

At each time step $t$, an activation set $\cK_{a_t}$ is realized from the collection $\{\cK_1, \dots, \cK_A\}$ according to a fixed distribution $\mathbf{p}$, where $p_a$ denotes the probability of selecting $\cK_a$. 
An allocation strategy $\phi$ observes the active set $\mathcal{K}_{a_t}$ and selects an available action $j \in \cK_{a_t}$. 
Formally, $\phi$ is a sequence of random variables $\{\phi(t)\}_{t \geq 1}$ where $\phi(t)$ represents the action chosen at time $t$ conditioned on the realized activation set $\cK_{a_t}$.

Let $N_{j,a}^\phi(T)$ be the total number of times action $j\in\cK_a$ is chosen by policy $\phi$, when set $\cK_a$ was active, up to time $T$. 
For each action, rewards are only obtained from the chosen action by the policy (side-observations do not contribute to the total reward).
The regret of policy $\phi$ at time $T$ for a fixed reward profile $\boldsymbol{\mu} = (\mu_1, \dots, \mu_K)$ and activation sets $\{\cK_1,\dots,\cK_A\}$, is defined as the expected difference between the reward of the optimal action in the active set and the chosen action:
\begin{align}\label{eq:regret_def}
    R_{\boldsymbol{\mu}}^\phi(T) = \mathbb{E} \left[ \sum_{t=1}^T \left( \mu^*_{a_t} - \mu_{\phi(t)} \right) \right] = \sum_{a=1}^A \sum_{j \in \cU_a} \Delta_{j,a} \E[N_{j,a}^\phi(T)],
\end{align}
where $\mu^*_a \triangleq \max_{k \in \mathcal{K}_a} \mu_k$ is the maximum mean reward in set $\mathcal{K}_a$, and $\Delta_{j,a} \triangleq \mu^*_a - \mu_j$ is the suboptimality gap of action $j\in\cK_a$. 
The objective is to find policies that minimize the growth rate of this regret as a function of time for every given fixed network $G$, activation sets $\{\cK_1,\dots,\cK_A\}$, and probability distribution $\mathbf{p} = \{p_1, \dots, p_A\}$. 
Henceforth, we drop the superscript $\phi$ unless it is required.


\section{LP Minimization Problem}
In this section, we describe the need for and define the LP minimization problem that we consider.
Figure \ref{fig:side_info_model_1} shows the different data acquisition pathways available in this framework. 
Consider the objective of estimating the mean reward for user 4. 
When activation set $\cK_1$ is realized, user 4 is active; thus, samples can be obtained directly by targeting user 4, or indirectly via side-observations by targeting user 1. Crucially, even when user 4 is unavailable (i.e., during activation set $\cK_2$), information can still be harvested by targeting its active neighbors, user 1 or user 2. This demonstrates that information gain is not strictly bound by an arm's instantaneous availability. By strategically leveraging the network topology and weighing the occurrence probabilities $\mathbf{p}$, our policy efficiently accumulates observations for specific arms through these side-channels, even when those arms are effectively `offline'.


To efficiently exploit the side-information structure, we require a sampling strategy that accounts for both network topology and activation set probabilities. Ideally, one would minimize cumulative regret; however, since sub-optimality gaps are unknown \textit{a priori}, we instead adopt a \textit{best-observation} strategy. We optimize for informational efficiency by minimizing the total sampling effort required to ensure every base-arm is sufficiently observed.

Inspired by \cite{buccapatnam2018reward}, we formulate the Linear Program (LP) $P_1$, where $z_{j,a}$ denotes the weight assigned to action $j$ specifically when set $\cK_a$ is active.
\begin{align}\label{eq:lp_problem_def}
     P_1:& \min_{\vz \geq 0} \sum_{a\in[A]}p_a \sum_{j\in\cK_a} z_{j,a} \\
     \text{s.t. }& \sum_{a\in[A]} p_a \sum_{j\in \cK_a \cap \cP_i} z_{j,a} \geq 1, \forall i\in\cN, \nonumber\\
     \text{and } & \sum_{j\in\cK_a} z_{j,a} \ge \epsilon, \forall a\in[A],\nonumber\\
     \text{and } & z_{j,a} \ge 0, \forall j\in\cK, \forall a\in[A] \nonumber.
\end{align} 
The objective function minimizes the total expected sampling frequency. The primary constraint imposes a global \textit{observability requirement}, ensuring that every base-arm $i$ accumulates at least unit mass of observation across all activation sets. A small constant $\epsilon > 0$ is included to guarantee non-zero exploration in every set. The optimal solution $\textbf{z}^* = \{(z^*_{j,a})_{j\in\cK_a}\}_{a\in[A]}$ is used for network-aware sampling in Algorithm \ref{alg:ucb-lp-new}.

\section{UCB-LP-A Policy}

In this section, we first outline the UCB-LP framework, which motivates our design. Subsequently, we detail the proposed UCB-LP-A policy (Algorithm \ref{alg:ucb-lp-new}).

\subsection{UCB-LP Explained}
UCB-LP-A builds on the UCB-LP policy \cite{buccapatnam2018reward}, which in turn extends the improved UCB framework from \cite{auer2010ucb}. 
While the original UCB-LP did not account for activation sets, we adapt its core structure to our setting using the modified LP-minimization problem (\ref{eq:lp_problem_def}).

In base UCB as in \cite{auer2010ucb}, the policy estimates the values of $\Delta_i$ in each round by a value $\tdel{m}$ which is initialized to 1 and halved in each round $m$.
This version of the UCB eliminates suboptimal actions at the end of each round $m$, where if the UCB score of an action $j$ is less than the LCB score of any active action (the actions not eliminated by round ($m-1$)), then action $j$ is marked as \textit{most likely not optimal} and hence deleted.
The elimination condition is given in the elimination subroutine in Algorithm \ref{alg:action-elimination}.
The core idea is that we need the policy to draw $n(m)$ observations for each active action to end round $m$, where $n(m)$ depends on the round $m$ and time $T$.



UCB-LP \cite{buccapatnam2018reward} adapts this elimination framework by leveraging side-information to satisfy the sample requirement $n(m)$. For each round $m$, the policy assigns a sampling budget of $z^*_j \Delta n_m$ to every action $j$, where $\Delta n_m = n(m) - n(m-1)$. The underlying LP constraints guarantee that this allocation yields at least $\Delta n_m$ observations for every base-arm. Consequently, sufficient statistics are acquired for all active actions, either directly or via side-information, often without requiring direct pulls for every candidate.



This static budgeting strategy, however, relies on the premise that all actions are always available. 
In our setting, the introduction of activation sets renders availability stochastic. 
We cannot simply mandate that an action $j$ be played $z^*_j \Delta n_m$ times consecutively, as the action may be inactive during the required window. 
Consequently, the direct application of UCB-LP is infeasible. 

\subsection{UCB-LP-A Explained}
The proposed UCB-LP-A policy is described in Algorithm \ref{alg:ucb-lp-new}, where Algorithm \ref{alg:action-elimination} serves as a sub-routine to perform the elimination step.
Similar to UCB-LP, there are two methods of achieving $\Delta n_m$ samples for each active action in round $m$ for any $\cK_a$ (the set of active actions for $\cK_a$ is denoted as $\cB_a$), -- 1) \textit{$z^*$-sampling} where we pull each action $j$ \textit{(not just active ones)} with probability $({z^*_{j,a}}/{\sum_{j\in\cK_a} z^*_{j,a}})$ when set $\cK_a$ is active until we get the required samples, and 2) \textit{uniform sampling} where we simply pull each \textit{active action} $\Delta n_m$ times.
Note that $z^*$-sampling exploits the side information structure and tries to choose highly-informative actions more often, while uniform sampling doesn't and treats each action as the same.
Further, $z^*$-sampling requires us to rely on samples possibly from other activation sets and hence forces us to solve for the sample requirement across activation sets together, while uniform sampling simply treats each set independently and does not have a global view.

For every round $m$, starting with $m=0$, for any $\cK_a$, we need to decide on the sampling method to be used.
There is a clear trade-off between these methods since $z^*$-sampling might be efficient at first; however, since the set of active actions, $\cB_a$, shrinks over time, uniform sampling will eventually become cheaper.
The proposed policy is meant to eliminate action $j\in\cK_a$ by the first round $m$ s.t. $2^{-m}<\Delta_{j,a}/2$, which means that we have a better chance of shortlisting the optimal action if we can run for more rounds within the same time horizon $T$.
Hence, we must choose the method that reduces the time steps required to finish round $m$ for each $\cK_a$.
However, since we want to minimize regret, instead of time steps, the question we must ask is - `Is the expected regret to finish round $m$ for all sets $\cK_a$ using $z^*$-sampling cheaper than uniform sampling?'

Observe that while using $z^*$-sampling, we rely on the samples generated by choosing actions in the other activation set.
Hence, if we perform elimination for $\cK_1$ while $\cK_2$ is still learning, then we break the dependency and cannot benefit from the side-information structure.
Therefore, $z^*$-sampling requires us to force every activation set to sync their local round counters $m_a$, and they perform elimination only when the sample requirement for each set is satisfied.

Hence, the algorithm can run under 2 different phases called the \textit{forced-sync phase} and the \textit{independent phase}.
The forced-sync phase corresponds to the case when we use $z^*$-sampling, where the local round counters $m_a$ are synced together, and each set performs the elimination step together.
The independent phase refers to a case where we simply apply uniform sampling and the local round counters $m_a$ can evolve independently for each set $\cK_a$.

\begin{algorithm}[!ht]
\caption{UCB-LP-A policy}
\begin{algorithmic}[1]
\STATE \textbf{Input:} Graph $G$, set of actions $\cK$, activation sets $\cK_a\: \forall a \in [A]$, its distribution $\textbf{p} = [p_1, p_2, ..., p_{A}]$, time horizon $T$, and optimal solution $\textbf{z}^*$ of the LP problem.
\STATE \textbf{Initialization:} Let $m_a = 0$, $\tilde{\Delta}_{a} = 1$, $\cB_{a}, \cB^{def}_a = \cK_a$, \textit{end\_round(a)}= \FALSE \:$\forall a \in [A]$, \textit{round\_sync}=\FALSE, $\hat{\textbf{f}} = [\hat{f}_j]_j$, $\textbf{T} = [T_j]_j$, with $\hat{f}_j,T_j=0\:\forall j \in \cK,\: v_{\min} = \min_i \left(\sum_a p_a \sum_{j \in \cK_a \cap \cP_i} \dfrac{z^*_{j,a}}{Z^*_a}\right)$.
\FOR{time $t = 1, 2, 3, ..., T$}

    \STATE Sample $a \in [A]$  with probability $p_a$, i.e., $\cK_a$ is active, set $m=m_{a}$. \COMMENT{\textit{Activation Set Selection}}

    \IF{\textit{round\_sync}=\FALSE}
    \STATE Set $\bar{\cB}=\{\cB_a\}_a$ and calculate $R=\sum_{i \in \bigcup_{a} \cB_a} R_i(\bar{\cB})$ using Equation \ref{eq:cal_r}.
    
    \IF[\textit{Phase Selection}]{$\left(\dfrac{1}{v_{\min}} \le 2\tdel{a} R \right)$} 
        \STATE Set \textit{round\_sync}=\TRUE.
    \ENDIF
    \ENDIF
    
    \IF[\textit{$m_{a}$ is off limits}]{$\left(m_{a} \ge \left\lfloor \frac{1}{2}\log_2 T  \right\rfloor\right)$}
        \STATE Select the current optimal action in $\cB_a$.
    \ELSE[\textit{$m_{a}$ is within limits}]
    
    \STATE Define $n(k) := \left\lceil \dfrac{2 \log(T \tilde{\Delta}_{k}^2)}{\tilde{\Delta}_{k}^2} \right\rceil$, where $\tdel{k}=2^{-k}$.
 
    \IF{$(|\cB_{a}| = 1)$}
        \STATE Pull the single action in $\cB_{a}$.
    \ELSIF[\textit{Forced-Sync}]{(\textit{round\_sync}=\TRUE)}
    \STATE Pull action $k \in \cK_{a}$ with $\mathbb{P}(k=j)= \left(\frac{z^*_{j,a}}{Z^*_{a}}\right)$, get rewards for base-arms $j\in\cC_k$, stitch together samples to get free sample for actions, and update $T_j,\:\forall j\in\cK$.


        \IF{$(T_j=n(m)\:\forall j \in \cK_b$ for some $b\in[A])$}
        \STATE Set $end\_round(b) = \TRUE,\:\forall$ such $b\in[A]$.
        \ENDIF
        


    \ELSE[\textit{Independent Phase}]
        \STATE Pull any action $k \in \cB^{def}_{a}$, get rewards for base-arms $j\in\cC_k$, stitch together samples to get free sample for actions and update $T_j,\forall j\in\cK$.
        \IF{$(T_j=n(m_b)$ for some $j\in \cC_k$ for some $b\in[A]$ with $j\in \cB^{def}_b)$}
        \STATE Delete $j$ from $\cB^{def}_b,\forall$ such $j,\forall$ such $b\in[A]$.
        \ENDIF
        \IF{($\cB^{def}_{b} = \varnothing$ for some $b\in[A]$)}
        \STATE Set $end\_round(b)=\TRUE,\:\forall$ such $b\in[A]$.
        \ENDIF
    \ENDIF

    
    
    \STATE Update mean estimate $\hat{f}_j,\:\forall j\in\cK$ using new samples and observation count $T_j$.

    \ENDIF

    \IF[\textit{Elim. - Forced-Sync}]{(\textit{round\_sync}=\TRUE)}
    \IF{(\textit{end\_round(b)} = \TRUE \:$\forall\: b\in[A]$)}
    \STATE Set \textit{round\_sync} = \FALSE.
    \FOR{each $b \in [A]$}
    \STATE $(m_b, \cB_b)$ = \textbf{Eliminate}($b, m_b,\cB_{b}, \hat{\textbf{f}}, \textbf{T}, T$), $\cB^{def}_{b}=\cB_{b}$, \textit{end\_round(b)}=\FALSE, $\tdel{a}=\tdel{a}/2$.
    \ENDFOR
    \ENDIF

    \ELSE[\textit{Elimination - Independent Phase}]
    \FOR{(each $b \in [A]$ s.t. $end\_round(b)$=\TRUE)}
    \STATE $(m_b, \cB_{b})$ = \textbf{Eliminate}($b, m_b,\cB_{b}, \hat{\textbf{f}}, \textbf{T}, T$), $\cB^{def}_{b}=\cB_{b}$, \textit{end\_round(b)}=\FALSE, $\tdel{a}=\tdel{a}/2$.
    \ENDFOR
    \ENDIF

\ENDFOR
\end{algorithmic}
\label{alg:ucb-lp-new}
\end{algorithm}

In line 4 of the algorithm, at every round $t$, we select an activation set $a\in[A]$ with probability $p_a$ and we operate in round $m_a$ for set $\cK_{a}$.
In lines 5-8, we decide the phase of operation.
During initialization, we calculated $v_{\min}$, which is the global rate for sample accumulation of the base-arm that has the slowest accumulation rate when using $z^*$-sampling.
Hence, $1/v_{\min}$ gives an upper bound of the number of pulls for the slowest base-arm, and hence every base-arm, to get at least 1 sample.
It also naively bounds the regret, normalized by $\Delta n_m$, for ending round $m$ for all $\cK_a$ in the forced-sync phase.
Now, for the independent phase, even if we don't exploit the side information structure, we still use the free samples, generated naturally, to update the mean estimates and observation counts.
These have to be taken into account for a realistic estimate of the regret for this phase.

Hence, we calculate $R(\Bar{\cB})=\sum_{i \in \bigcup_a{\{\cB_a \in \bar{\cB}}\}} R_i(\bar{\cB})$, where $R_i(\bar{\cB})$ is given in Equation \ref{eq:cal_r}, where $P_{\dir}(i,\bar{\cB})$ and $P_{\free}(i,\bar{\cB})$ represents the probability of getting a sample for action $i$ in the current time step via directly choosing that action and getting a free side-information sample respectively.
Here $\bar{\cB}=\{\cB_1,...,\cB_A\}$, i.e., the set of set of active actions for all $\cK_a$.
\begin{align}
    &R_i(\bar{\cB}) = \frac{P_{\text{dir}}(i,\bar{\cB})}{P_{\text{dir}}(i,\bar{\cB}) + P_{\text{free}}(i,\bar{\cB})}, \text{with } P_{\text{dir}}(i,\bar{\cB}) = \sum_{a : i \in \cB_a} \frac{p_a}{|\cB_a|}, \nonumber \\
    &\text{and } P_{\text{free}}(i,\bar{\cB}) = \sum_{a=1}^{A} \sum_{j \in \cB_a \setminus \{i\}} \frac{p_a}{|\cB_a|} \cdot \mathds{1}\left(\cF_i \subseteq \cC_j\right). \label{eq:cal_r}
\end{align}
$R$ is the effective number of pulls for active actions to ensure 1 sample for each active action for all $\cK_a$.
Using $R$, we can upper bound the regret to complete round $m_{a}$ for all $\cK_a$ by $2\tdel{a}R$ (more details in proof of Theorem \ref{thm:new_ucb_lp_result} in Appendix \ref{app:proof_thm_1}).
Hence, when $\left(\frac{1}{v_{\min}} \le 2\tdel{a}R\right)$, we operate under forced-sync phase with \textit{round\_sync}=\textbf{true} and otherwise choose independent phase with \textit{round\_sync}=\textbf{false}.

Lines 11-33 form the core of the algorithm.
In lines 13-14, we pull the only remaining action for the given $\cK_a$ if $\cB_a$ is a singleton set.
Else, in 15-18, we operate under the forced-sync phase, where after choosing a particular action $k\in\cK_a$ using $z^*$-sampling, we get samples for base-arms $j\in\cC_k$.
We stich together the samples collected earlier and now via side information to generate a sample, if possible, for actions and update their observation count. 
If all the active actions for any set $\cK_b$ has met the requirement to get to $n(m)$ samples, we mark those sets as done by \textit{end\_round}=\textbf{true}.
In lines 26-30, if each $\cK_a$ has met the sample requirement, we perform the elimination step using the \textbf{Eliminate} subroutine (Algorithm \ref{alg:action-elimination}).
Similarly, lines 20-24 deal with using an independent phase where we maintain a set $\cB^{def}_a \subseteq \cB_a$, which consists of the active arms that have a \textit{deficit} to get to $n(m)$ samples.
When we pull an action $k\in\cB^{def}_a$, we observe the rewards for its base-arms and possibly generate a free side information sample for other actions as we did for the forced-sync phase.
As as action $j$ gets to $n(m_b)$ sample for any set $\cK_b$, we remove it from the corresponding $\cB^{def}_b$.
Lines 31-33 deal with the elimination step for this phase.
After every elimination step, we halve the current $\tdel{a}$ and increase round counter $m_a$ by 1.

\begin{algorithm}
\caption{Action Elimination Subroutine (Eliminate)}
\begin{algorithmic}[1]
    \STATE \textbf{Input:} Round $m$, set of active actions $\cB$, empirical means $\hat{\textbf{f}}$, observation count $\textbf{T}$, time horizon $T$.
    \STATE To get $\cB_{new}$, delete all actions $j$ in $\cB$ for which
    $$\hat{f}_j + \sqrt{\frac{\log(T \tilde{\Delta}_{m}^2)}{2 T_j}} < \max_{i \in \cB} \left\{ \hat{f}_i - \sqrt{\frac{\log(T \tilde{\Delta}_{m}^2)}{2 T_i}} \right\}$$
    \STATE $m_{new} = m + 1$ \:\COMMENT{\textit{increment the local round counter}}
    \RETURN $(m_{new}, \cB_{new})$.
\end{algorithmic}
\label{alg:action-elimination}
\end{algorithm}

\section{Theoretical Result}
\begin{assumption}\label{as:delta}
    $\Delta_{j,a} > \frac{2}{T},\:\forall j\in\cK_a,\:\forall a\in[A]$ if the time horizon to run the algorithm is $T$.
\end{assumption}

\begin{theorem}\label{thm:new_ucb_lp_result}
    For all action $j\in\cK_a,\forall a\in[A]$, define round $m_{j,a} := \min\{m \in \cM: \tilde{\Delta}_m <\Delta_{j,a}/2\}$, $\cG_{m,a} = \{j\in\cK_a:m_{j,a}\ge m\}$ and $\bar{\cG}_m=\{\cG_{m,1},...,\cG_{m,A}\}$.
    and $\bar{m} := \max \Bigl\{ m \in \cM: \frac{1}{v_{\min}} \le 2 \tdel{m} \sum_{j \in \bigcup_a\cG_{m,a}} 
    R_j(\bar{\cG}_m)\Bigl\}$
    where $v_{\min} = \min_{i\in\cN} \left(\sum_a p_a \sum_{j \in \cK_a \cap \cP_i} \frac{z^*_{j,a}}{Z^*_a}\right)$, $\cM= \{0,1,...,\left\lfloor \frac{1}{2}\log_2T \right\rfloor-1\}$, 
    Also, define set $\cD_a := \{ j\in\cU_a: m_{j,a} > \bar{m} \}$.
    Then, given Assumption \ref{as:delta} holds, the expected regret of the UCB-LP-A policy as given in Algorithm \ref{alg:ucb-lp-new}, denoted by $\bE[R(T)]$, is upper bounded by: 
\begin{align*}
    &\sum_{a=1}^A \Biggr[ \sum_{j \in \cU_a \setminus \cD_a}\gamma_{j,a} \left(\frac{2\log(T\tdel{\bar{m}}^2)}{\tdel{\bar{m}}^2} \right) \Delta_{j,a} +  \sum_{j \in \cD_a}  \Biggl\{ (\gamma_{j,a}-1).\quad\\
    &\left( \frac{2 \log(T \tilde{\Delta}_{\bar{m}}^2)}{\tdel{\bar{m}}^2} \right) + \left( \frac{2 \log(T \tilde{\Delta}_{m_{(j,a)}}^2)}{\tilde{\Delta}_{m_{(j,a)}}^2} \right) \Biggl\} \Delta_{j,a} \Biggl]\:\: +\:\: O(K)
\end{align*}
where $\gamma_{j,a} = \left(\dfrac{p_az^*_{j,a}Z^*_{\max}}{Z^*_a}\right)$ with $Z^*_a = \sum_{j\in\cK_a} z^*_{j,a}$ and $Z^*_{\max} = \max_a Z^*_a$. The $O(K)$ term captures constants independent of time and is upper bounded by (where $\Delta_a^{\min}:= \min_j \del{j,a}$):
\begin{align*}
    O(K) \le& \sum_{a=1}^A \sum_{j \in \cU_a} (\gamma_{j,a} + 1) \Delta_{(j,a)}\quad +\\ 
    \left(\max_{j,a} \Delta_{j,a}\right)& \left[ \sum_{a=1}^A \frac{32\:|\cU_a|}{3(\Delta_a^{\min})^2} + \sum_{a=1}^A \sum_{j \in \cU_a} \frac{32}{\Delta_{(j,a)}^2} \right]
\end{align*}
\end{theorem}

\begin{proof}(Sketch)
    We calculate the conditional regret under good and bad event, where the good event is the case where, with high probability, each suboptimal action $j$ for each set $\cK_a$, is eliminated on or before round $m_{j,a}$. Some actions observe only the forced-sync phase, while some observe both the phases, and we can calculate regret contributions accordingly.
    See Appendix \ref{app:proof_thm_1} for full proof.
\end{proof}

\begin{theorem}(Baseline)\label{thm:baseline_result}
    For any action $j\in\cK_a$  for any $a\in [A]$, define round $m_{j,a} := \min\{m \in \cM: \tilde{\Delta}_m <\Delta_{j,a}/2\}$, where $\cM= \{0,1,...,\left\lfloor \frac{1}{2}\log_2T \right\rfloor-1\}$.
    Then, given Assumption \ref{as:delta} holds, the expected regret of the baseline UCB policy, denoted by $\bE[R(T)]$, is upper bounded by: 
\begin{align*}
     \sum_{a=1}^A \sum_{j \in \cU_a} \left( \frac{2 \log(T \tilde{\Delta}_{m_{(j,a)}}^2)}{\tilde{\Delta}_{m_{(j,a)}}^2} \right) \Delta_{j,a} + O(K)
\end{align*}
The $O(K)$ term captures constants independent of time and is upper bounded by (where $\Delta_a^{\min}:= \min_j \del{j,a}$):
\begin{align*}
    O(K) \le& \sum_{a=1}^A \sum_{j \in \cU_a}\Delta_{j,a} +\\
    \left(\max_{j,a} \Delta_{j,a}\right) &\left[ \sum_{a=1}^A \frac{32\:|\cU_a|}{3(\Delta_a^{\min})^2} + \sum_{a=1}^A \sum_{j \in \cU_a } \frac{32}{\Delta_{j,a}^2} \right]
\end{align*}
\end{theorem}
\begin{proof}(Sketch)
    Similar proof as Theorem \ref{thm:new_ucb_lp_result} with the same good and bad events, but we do not have the forced-sync phase and simply calculate for the independent phase. See Appendix \ref{app:proof_thm_2} for full proof.
\end{proof}

Theorem \ref{thm:new_ucb_lp_result} is the main theoretical result of the paper where we present the upper bound of the expected regret for the proposed UCB-LP-A policy.
Theorem \ref{thm:baseline_result} presents the upper bound for the baseline algorithm, where we apply the UCB with elimination for each set independently and not exploit the side-information structure.
In Theorem \ref{thm:new_ucb_lp_result}, we can observe that the actions in $\cU_a\setminus\cD_a$ and $\cD_a$ contribute differently towards regret. 
The actions in $\cD_a$ are the ones that experience both the forced-sync and independent phase, while actions in $\cU_a\setminus\cD_a$ experience only the forced-sync phase.

\begin{figure*}[h!] 
    \centering
    \begin{subfigure}[b]{0.32\linewidth}
        \centering
        \includegraphics[width=\linewidth]{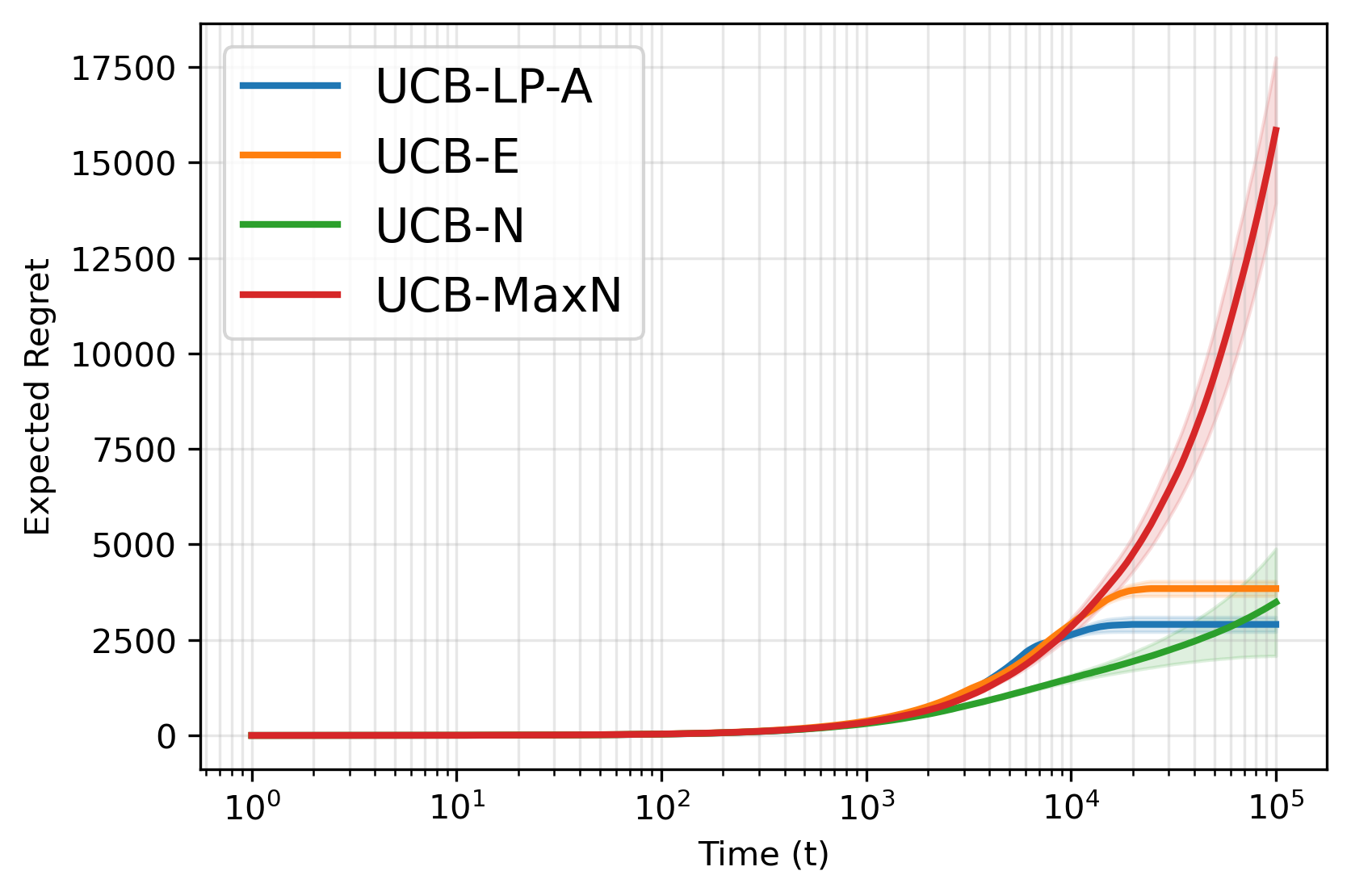}
        \caption{K=100, Activation Sets=2, Optimal=5.}
        \label{fig:sim_1}
    \end{subfigure}
    \hfill 
    \begin{subfigure}[b]{0.32\linewidth}
        \centering
        \includegraphics[width=\linewidth]{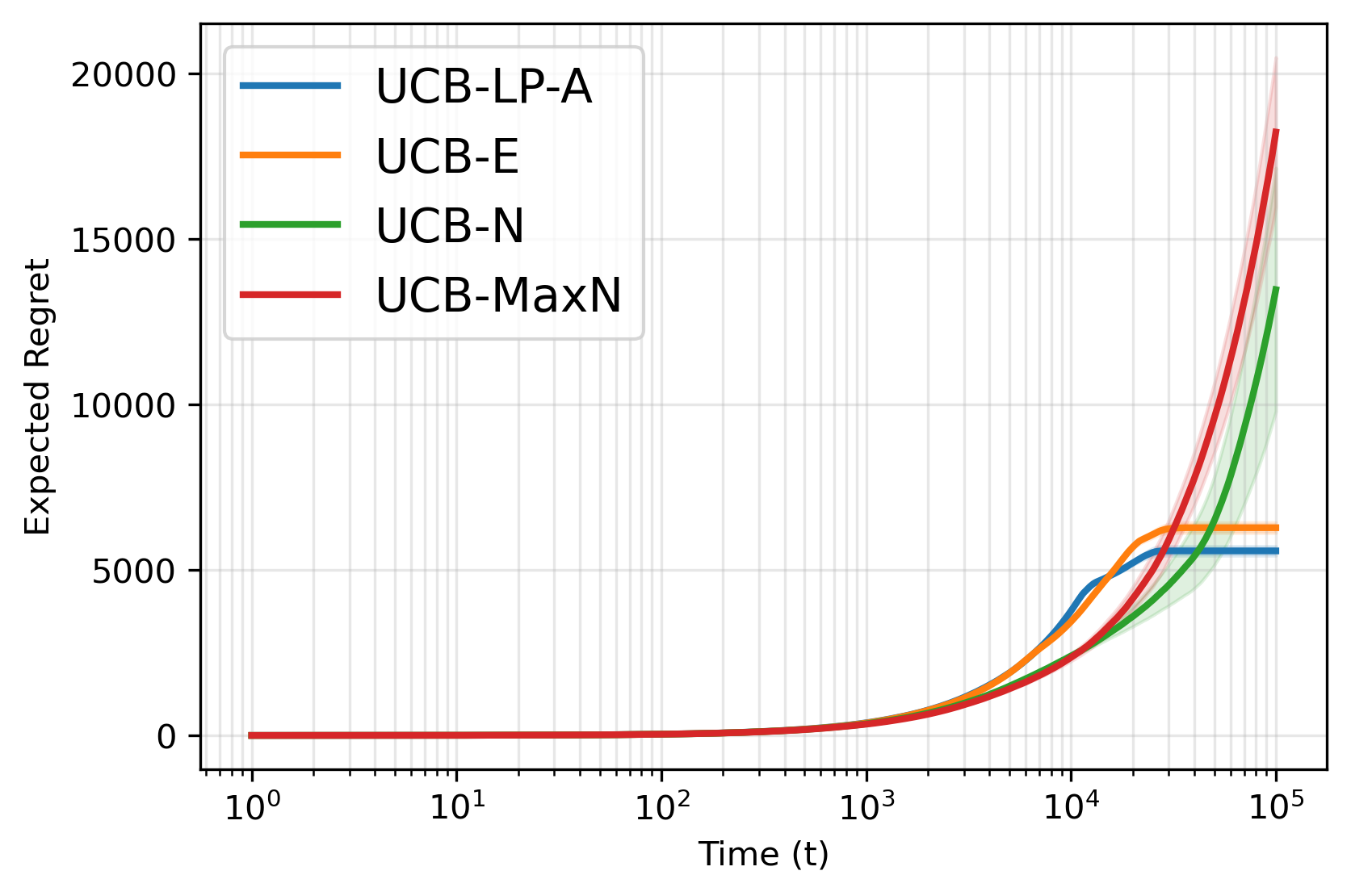}
        \caption{K=200, Activation Sets=3, Optimal=10.}
        \label{fig:sim_2}
    \end{subfigure}
    \hfill
    \begin{subfigure}[b]{0.32\linewidth}
        \centering
        \includegraphics[width=\linewidth]{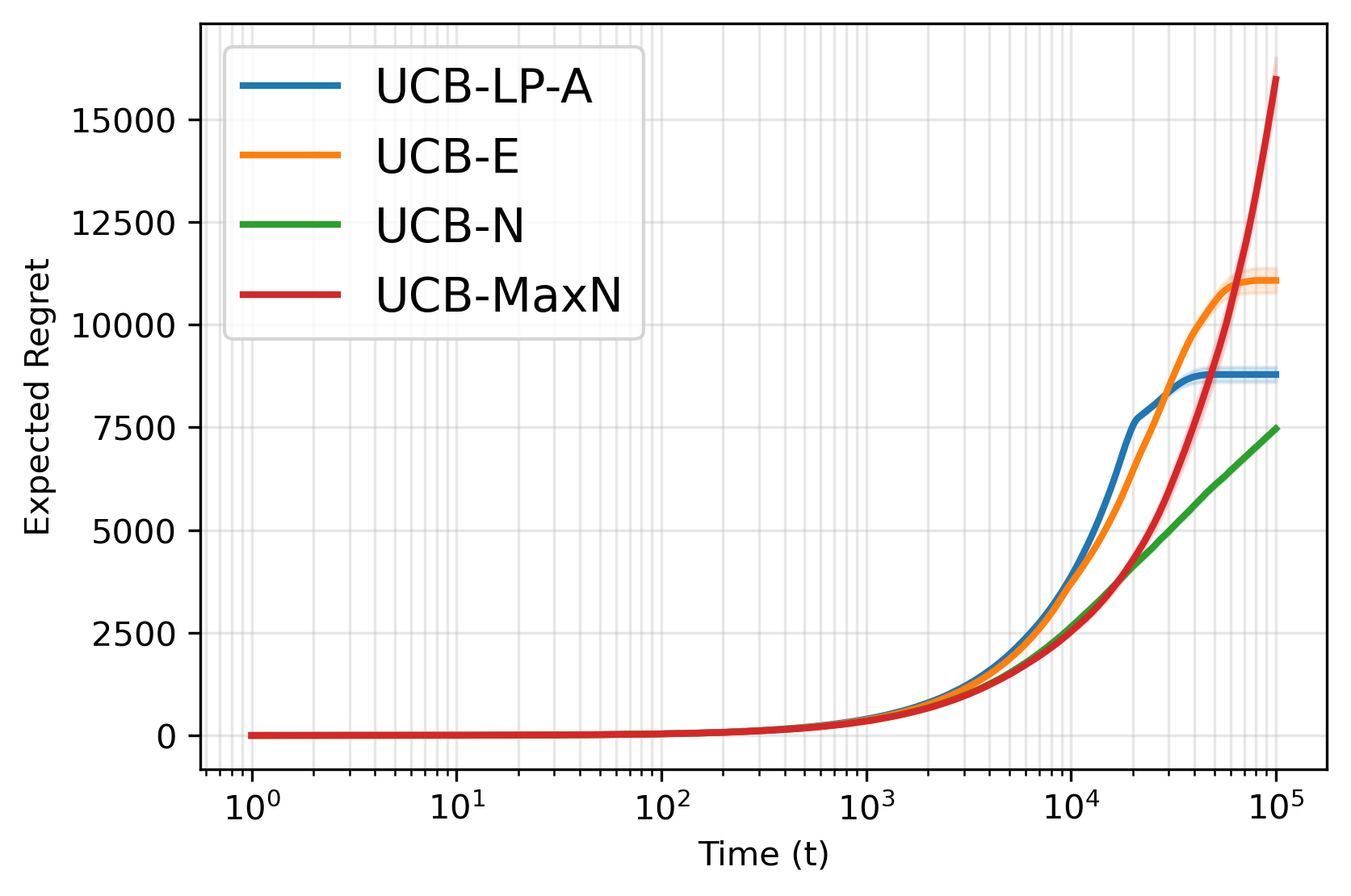}
        \caption{K=300, Activation Sets=5, Optimal=20.}
        \label{fig:sim_3}
    \end{subfigure}
    \caption{Regret comparison for a simulated social network generated using Barabási–Albert (BA) model with $m=3$.}
    \label{fig:sim_social}
\end{figure*}

Observe that $\gamma_{j,a}$ takes into account both the set probabilities and the optimal solution $z^*_{j,a}$. 
Observe that for actions $j\in\cD_a$, the $\left(\dfrac{2 \log(T \tilde{\Delta}_{m_{(j,a)}}^2)}{\tilde{\Delta}_{m_{(j,a)}}^2} \Delta_{j,a} \right)$ term in the second inner summation in Theorem \ref{thm:new_ucb_lp_result} is the same as the term in Theorem \ref{thm:baseline_result} for corresponding $j\in\cU_a$.
This is the maximum contribution of regret for this action $j$ when not using any side-information structure.
Hence, the term with $(\gamma_{j,a}-1)$ would not appear if not using $z^*$-sampling.
For cases with great side-information structure, $(\gamma_{j,a}-1)<0$, which means that the term with $(\gamma_{j,a}-1)$ only potentially reduces the coefficient of regret for actions that have survived the forced-sync phase.
Hence, this shows the benefit of exploiting the side-information structure.


\section{Numerical Results}

We benchmark UCB-LP-A against three baselines: UCB with elimination (UCB-E) \cite{auer2010ucb}, UCB-N, and UCB-MaxN \cite{caron2012leveraging}. 
We adapt these policies to the current setting by applying them to each activation set while updating observation counts and mean estimates using side-information. 
The adaptation for UCB-E and UCB-N is direct. 
UCB-MaxN, however, requires a constrained selection rule. 
It targets the arm $i$ with the highest UCB index but plays the best neighbor $j$ found in the intersection of $i$'s neighborhood and the active set. 
In the results, solid curves denote mean cumulative regret over 20 independent trials, and shaded regions indicate 95\% confidence intervals. $\epsilon=10^{-5}$ is used across experiments.



\subsection{Social Network}


The UCB-LP-A framework applies directly to social network targeting. 
Here, activation sets represent the cohort of users online at any given instance, and the action set equals the set of base-arms, both being users. 
Rewards are Bernoulli distributed with mean $\mu_i$, representing the probability of accepting a promotion. Selecting a user reveals their own outcome along with side-observations from all 1-hop neighbors. By defining $\cF_j=\{j\}$ and $f_j(\textbf{X}_j(t))=X_j(t)$, the reward structure ensures that cumulative regret is driven solely by the response of the targeted user.

\subsubsection{Simulated Social Network}

To evaluate performance in a realistic setting, we simulate a social network using the Barabási–Albert (BA) model \cite{barabasi1999emergence}, which generates scale-free networks with a power-law degree distribution ($P(k) \sim k^{-3}$). 
This topology captures real-world preferential attachment dynamics, allowing us to test how algorithms leverage side-information from both `hubs' and peripheral users. 
We generated networks of 100, 200, and 300 users with attachment parameter $m=3$ (Figure \ref{fig:sim_1}). 
In each scenario, a subset of users (denoted as ``Optimal=X'' in captions) was assigned a high mean reward $\mu_i=0.9$, while others were drawn uniformly from $[0.3, 0.7]$. 
Activation sets were sampled uniformly ($\textbf{p}$ is uniform).

Figures \ref{fig:sim_1} through \ref{fig:sim_3} illustrate the cumulative regret. 
In all cases, UCB-LP-A and UCB-E reach a plateau, indicating they successfully identified the optimal actions, whereas UCB-N and UCB-MaxN continue to incur regret. 
Notably, UCB-MaxN performs worse than UCB-N. 
This is the opposite of the findings in \cite{caron2012leveraging}. 
The degradation arises because UCB-MaxN's neighbor selection is constrained by the activation sets; the highest-UCB neighbor is frequently inactive. 
Crucially, UCB-LP-A consistently outperforms UCB-E, confirming that our LP-based policy effectively exploits side-information to accelerate learning beyond standard elimination strategies.

\begin{figure}[h!]
    \centering
    
    \begin{subfigure}[b]{\linewidth} 
        \centering
        \includegraphics[width=\linewidth]{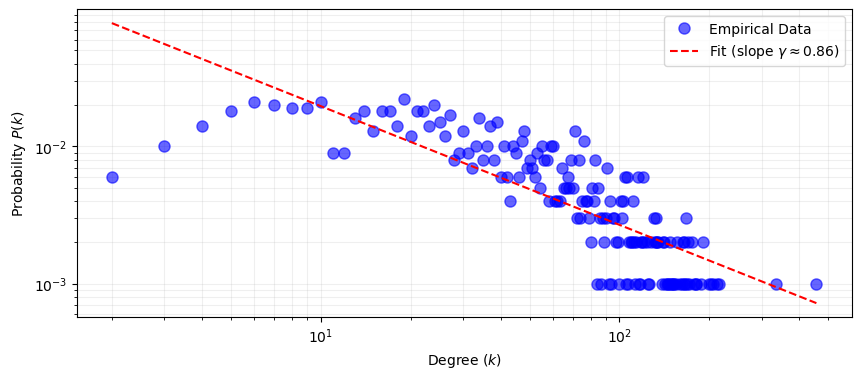}
        \caption{log-log plot for degree distribution of subgraph.}
        \label{fig:facebook_1}
    \end{subfigure}
    \par
    \begin{subfigure}[b]{\linewidth}
        \centering
        \includegraphics[width=\linewidth]{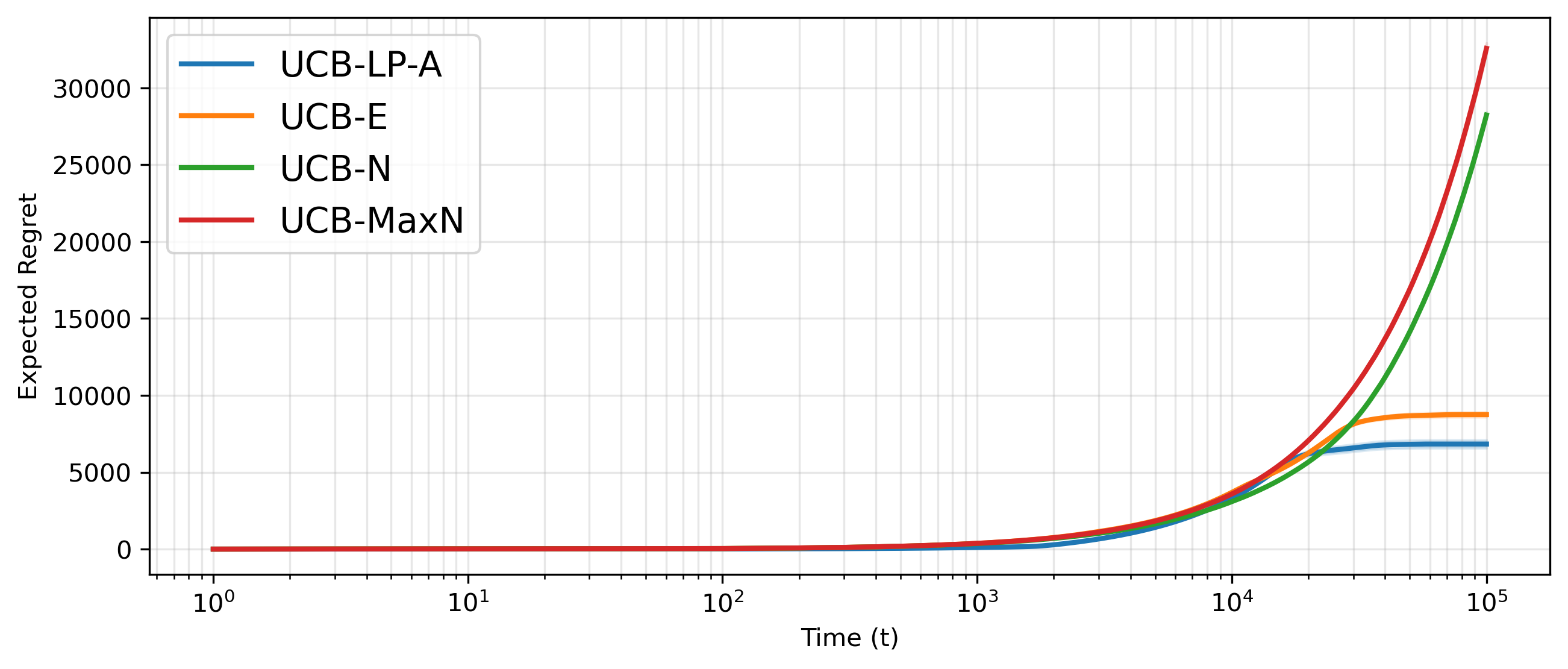}
        \caption{Regret comparison of all the policies.}
        \label{fig:facebook_2}
    \end{subfigure}
    \caption{Regret comparison for 1000 user subgraph from  Facebook Dataset.}
    \label{fig:facebook}
\end{figure}

\subsubsection{Facebook Dataset}

To validate performance on real-world topology, we employ the \textit{ego-Facebook} dataset from SNAP \cite{snapnets}. 
We extracted a 1000-node subgraph via random walk sampling to maintain structural integrity. 
Figure \ref{fig:facebook_1} confirms that the subgraph retains a power-law degree distribution, evidenced by the near-linear log-log plot. 
We partitioned users into 10 equiprobable, disjoint activation sets. 
Rewards were assigned with $\mu_i=0.9$ for 50 optimal users and $\mu_i \sim \mathcal{U}[0.3, 0.7]$ for the remaining 950. 
Figure \ref{fig:facebook_2} demonstrates that UCB-LP-A and UCB-E plateau, indicating successful identification of optimal users. 
Conversely, UCB-N and UCB-MaxN continue to accumulate regret. Consistent with the synthetic simulations, UCB-LP-A yields the superior performance.




\begin{figure}[h!]
    \centering
    \begin{subfigure}[b]{\linewidth} 
        \centering
        \includegraphics[width=0.65\linewidth]{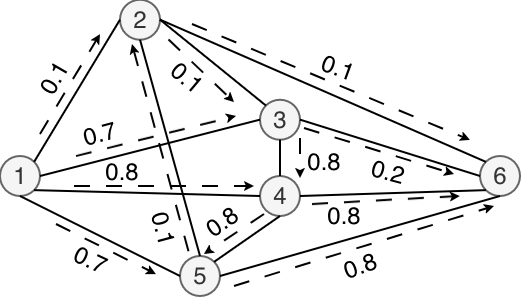}
        \caption{Example network for routing problem.}
        \label{fig:comm_1}
    \end{subfigure}
    \par
    \begin{subfigure}[b]{\linewidth}
        \centering
        \includegraphics[width=\linewidth]{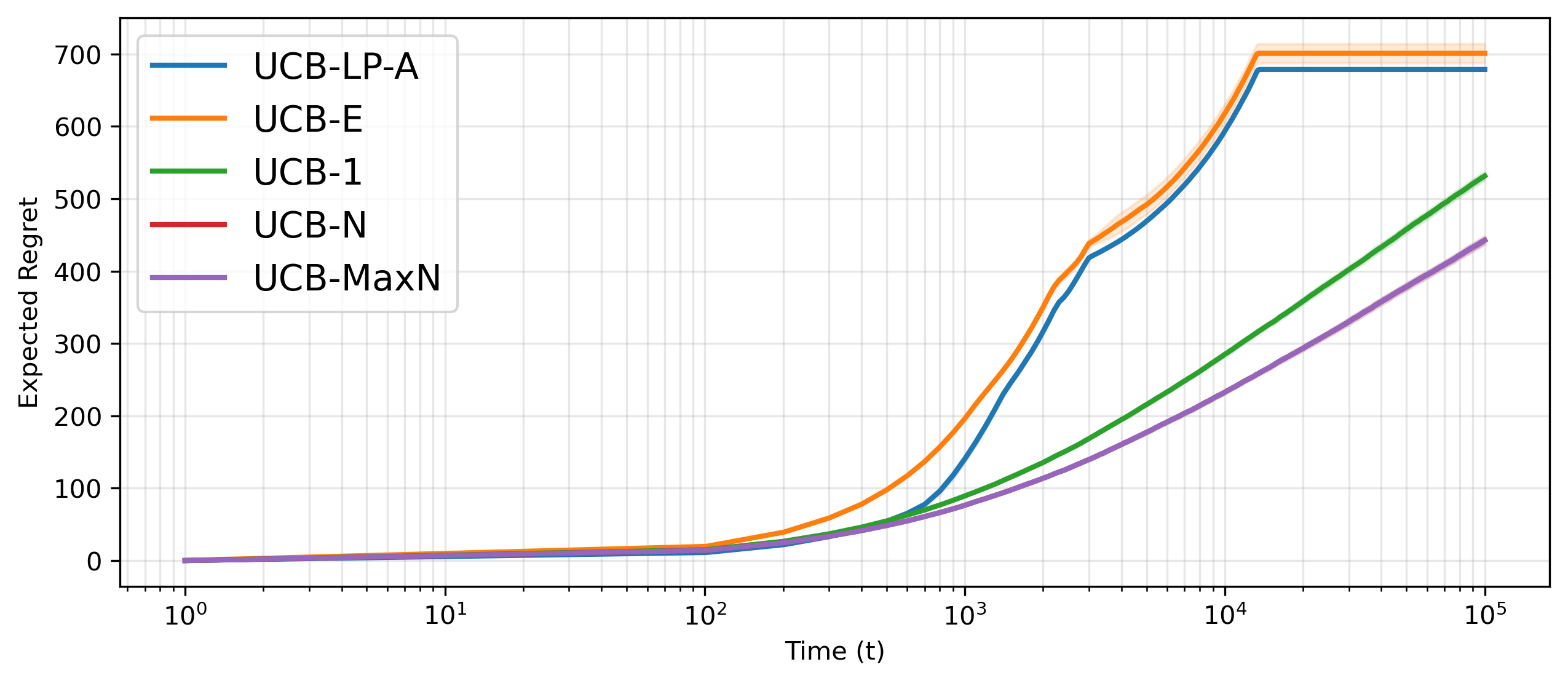}
        \caption{Regret comparison of all policies.}
        \label{fig:comm_2}
    \end{subfigure}
    \caption{Example routing problem and regret comparison.}
    \label{fig:vertical_stack}
\end{figure}

\subsection{Communication Network Example}

We simulate a routing problem on a 6-node network with 12 directed links and 8 paths from source (node 1) to destination (node 6), as detailed in Figure \ref{fig:comm_1}. 
Links represent base-arms ($N=12$), and paths represent actions ($K=8$). 
Selecting a path reveals individual link delays, offering side-information for overlapping paths. 
Two equiprobable activation sets are defined by distinct link failure scenarios: the inner cycle $\{(2,3), (3,4), (4,5), (5,2)\}$ fails in set 1, while the direct edges to the sink $\{(2,6), (3,6), (4,6), (5,6)\}$ fail in set 2. 
We define the reward as $f_j(\textbf{X}_j(t)) = 1 - \sum_{i\in\cC_j}X_i(t)/B$, with $B=5$ and Bernoulli link delays. 
Partial link observations are aggregated to construct valid side-samples for various paths whenever their full link set is observed. 
We additionally benchmark against UCB-1 \cite{auer2002finite}, which ignores side-information, to quantify the gain from structural exploitation.


Figure \ref{fig:comm_2} indicates that both UCB-LP-A and UCB-E plateau, confirming convergence to the optimal path per activation set. 
However, UCB-LP-A achieves significantly lower regret by effectively exploiting cross-set side-information. 
UCB-N and UCB-MaxN have almost overlapping curves since, in this case, the performance of UCB-MaxN degrades to that of UCB-N, since there is no non-trivial clique (clique with more than one element) in this problem.




\section{Conclusion and Future Work}
In this work, we presented a novel framework for Multi-Armed Bandits that integrates graph-based side-observations with stochastic action availability. 
We proposed the UCB-LP-A policy, demonstrating that an LP-based sampling approach can efficiently optimize exploration. 
As a preliminary study, our current analysis relies on the assumption of known activation sets and their occurrence probabilities. 
A critical direction for future work is to relax this constraint, extending the framework to settings where availability statistics are unknown. 
This would require developing algorithms capable of learning the underlying availability patterns online alongside the reward distributions, making the policy robust to fully dynamic and unknown environments.

\bibliographystyle{IEEEtran}
\bibliography{bibliography}

@article{auer2010ucb,
  title={UCB revisited: Improved regret bounds for the stochastic multi-armed bandit problem},
  author={Auer, Peter and Ortner, Ronald},
  journal={Periodica Mathematica Hungarica},
  volume={61},
  number={1-2},
  pages={55--65},
  year={2010},
  publisher={Akad{\'e}miai Kiad{\'o}, co-published with Springer Science+ Business Media BV~…}
}

@article{buccapatnam2018reward,
  title={Reward maximization under uncertainty: Leveraging side-observations on networks},
  author={Buccapatnam, Swapna and Liu, Fang and Eryilmaz, Atilla and Shroff, Ness B},
  journal={Journal of Machine Learning Research},
  volume={18},
  number={216},
  pages={1--34},
  year={2018}
}

@article{barabasi1999emergence,
  title={Emergence of scaling in random networks},
  author={Barab{\'a}si, Albert-L{\'a}szl{\'o} and Albert, R{\'e}ka},
  journal={science},
  volume={286},
  number={5439},
  pages={509--512},
  year={1999},
  publisher={American Association for the Advancement of Science}
}

@article{caron2012leveraging,
  title={Leveraging side observations in stochastic bandits},
  author={Caron, St{\'e}phane and Kveton, Branislav and Lelarge, Marc and Bhagat, Smriti},
  journal={arXiv preprint arXiv:1210.4839},
  year={2012}
}

@article{auer2002finite,
  title={Finite-time analysis of the multiarmed bandit problem},
  author={Auer, Peter and Cesa-Bianchi, Nicolo and Fischer, Paul},
  journal={Machine learning},
  volume={47},
  number={2},
  pages={235--256},
  year={2002},
  publisher={Springer}
}

@misc{snapnets,
  author       = {Jure Leskovec and Andrej Krevl},
  title        = {{SNAP Datasets}: {Stanford} Large Network Dataset Collection},
  howpublished = {\url{http://snap.stanford.edu/data}},
  month        = jun,
  year         = 2014
}

@article{lai1985asymptotically,
  title={Asymptotically efficient adaptive allocation rules},
  author={Lai, Tze Leung and Robbins, Herbert},
  journal={Advances in applied mathematics},
  volume={6},
  number={1},
  pages={4--22},
  year={1985},
  publisher={Academic Press}
}

@article{mannor2011bandits,
  title={From bandits to experts: On the value of side-observations},
  author={Mannor, Shie and Shamir, Ohad},
  journal={Advances in Neural Information Processing Systems},
  volume={24},
  year={2011}
}

@inproceedings{pandey2007multi,
  title={Multi-armed bandit problems with dependent arms},
  author={Pandey, Sandeep and Chakrabarti, Deepayan and Agarwal, Deepak},
  booktitle={Proceedings of the 24th international conference on Machine learning},
  pages={721--728},
  year={2007}
}

@article{bubeck2011x,
  title={X-Armed Bandits.},
  author={Bubeck, S{\'e}bastien and Munos, R{\'e}mi and Stoltz, Gilles and Szepesv{\'a}ri, Csaba},
  journal={Journal of Machine Learning Research},
  volume={12},
  number={5},
  year={2011}
}

@article{rusmevichientong2010linearly,
  title={Linearly parameterized bandits},
  author={Rusmevichientong, Paat and Tsitsiklis, John N},
  journal={Mathematics of Operations Research},
  volume={35},
  number={2},
  pages={395--411},
  year={2010},
  publisher={INFORMS}
}

@inproceedings{li2010contextual,
  title={A contextual-bandit approach to personalized news article recommendation},
  author={Li, Lihong and Chu, Wei and Langford, John and Schapire, Robert E},
  booktitle={Proceedings of the 19th international conference on World wide web},
  pages={661--670},
  year={2010}
}

@inproceedings{chen2013combinatorial,
  title={Combinatorial multi-armed bandit: General framework and applications},
  author={Chen, Wei and Wang, Yajun and Yuan, Yang},
  booktitle={International conference on machine learning},
  pages={151--159},
  year={2013},
  organization={PMLR}
}

@inproceedings{buccapatnam2014stochastic,
  title={Stochastic bandits with side observations on networks},
  author={Buccapatnam, Swapna and Eryilmaz, Atilla and Shroff, Ness B},
  booktitle={The 2014 ACM international conference on Measurement and modeling of computer systems},
  pages={289--300},
  year={2014}
}

@article{kleinberg2010regret,
  title={Regret bounds for sleeping experts and bandits},
  author={Kleinberg, Robert and Niculescu-Mizil, Alexandru and Sharma, Yogeshwer},
  journal={Machine learning},
  volume={80},
  number={2},
  pages={245--272},
  year={2010},
  publisher={Springer}
}

@inproceedings{kanade2009sleeping,
  title={Sleeping experts and bandits with stochastic action availability and adversarial rewards},
  author={Kanade, Varun and McMahan, H Brendan and Bryan, Brent},
  booktitle={Artificial Intelligence and Statistics},
  pages={272--279},
  year={2009},
  organization={PMLR}
}

@article{bnaya2013volatile,
  title={Volatile Multi-Armed Bandits for Guaranteed Targeted Social Crawling.},
  author={Bnaya, Zahy and Puzis, Rami and Stern, Roni and Felner, Ariel},
  journal={AAAI (Late-Breaking Developments)},
  volume={2},
  number={2.3},
  pages={16--21},
  year={2013}
}

@article{chakrabarti2008mortal,
  title={Mortal multi-armed bandits},
  author={Chakrabarti, Deepayan and Kumar, Ravi and Radlinski, Filip and Upfal, Eli},
  journal={Advances in neural information processing systems},
  volume={21},
  year={2008}
}

@article{li2019combinatorial,
  title={Combinatorial sleeping bandits with fairness constraints},
  author={Li, Fengjiao and Liu, Jia and Ji, Bo},
  journal={IEEE Transactions on Network Science and Engineering},
  volume={7},
  number={3},
  pages={1799--1813},
  year={2019},
  publisher={IEEE}
}

@inproceedings{cohen2016online,
  title={Online learning with feedback graphs without the graphs},
  author={Cohen, Alon and Hazan, Tamir and Koren, Tomer},
  booktitle={International Conference on Machine Learning},
  pages={811--819},
  year={2016},
  organization={PMLR}
}
\clearpage
\onecolumn

\appendices
\section{Proof of Theoreom \ref{thm:new_ucb_lp_result}}\label{app:proof_thm_1}
\begin{proof}
Let $i^*_a$ and $\mu^*_a$ denote the optimal action and the mean reward of the optimal action for the activation set $\cK_a$.
Let the mean reward gap for any suboptimal action $j\in \cK_a$ for any $a\in [A]$, be denoted as $\Delta_{j,a} = (\mu^*_a - \mu_j)$.
Also, denote $\cU_a=\cK_a\setminus i^*_a$ as the set of suboptimal actions for set $\cK_a$.
To ensure each active action is pulled at least once during round $m$, we require $n(m) \ge 1$ for a given time horizon $T$. Based on the definition of $n(m)$, this condition is satisfied for all $m \in \cM = \{0, 1, \dots, \lfloor \frac{1}{2}\log_2 T \rfloor - 1\}$. Hence, for each activation set $\cK_a$, the local round counter $m_a\in\cM$. 

The set of active arms at the start of round $m$ for activation set $\cK_a$ is denoted by $\cB_{m,a}$.
Also, let $\bar{\cB}_m=\{\cB_{m,1},...,\cB_{m,A}\}$ be the set containing the set of active actions at the start of round $m$ for all activation sets $\cK_a$.
If time $T$ does not run out after the end of round $m_a=\lfloor \frac{1}{2}\log_2 T \rfloor - 1$ for set $\cK_a$, for the next round $m$, every time set $\cK_a$ appears, we simply choose the current optimal action from active set $\cB_{m,a}$.


For each suboptimal action $j\in \cU_a$, define round $m_{j,a} := \min\{m\in \cM: \tilde{\Delta}_m <\Delta_{j,a}/2\}$, $\tilde{\Delta}_m = 2^{-m}$.
Under Assumption \ref{as:delta}, $\exists m \in\cM$ s.t. $\tilde{\Delta}_m <\Delta_{j,a}/2, \forall j \in\cK_a, \forall a\in[A]$. 
For optimal action $i^*_a$ of every set $\cK_a$, $m_{j,a}=\infty$ by convention.
Then by definition of $m_{j,a}$, for all sets $\cK_a$ for all rounds $m_a < m_{j,a}, \Delta_{j,a} \le 2 \tilde{\Delta}_{m_a}$ and

\begin{equation}\label{eq:delta_conditions}
    \frac{2}{\del{j,a}} < 2^{m_{j,a}} = \frac{1}{\tdel{m_{j,a}}} \le \frac{4}{\del{j,a}} < \frac{1}{\tdel{m_{(j+1,a)}}} = 2^{m_{(j+1,a)}}.
\end{equation}
Also, lets define $\ucb_j(m)$ and $\lcb_j(m)$ which are functions for action $j$ calculated at the end of round $m$, as the following, where $W(k,m) = \sqrt{\frac{\log(T\tdel{m}^2)}{2k}}$:
\begin{equation}\label{eq:ucb_lcb_def}
    \ucb_j(m) = \hat{f}_j(m) + W(T_j(m),m)\quad\text{and}\quad\: \lcb_j(m) = \hat{f}_j(m) - W(T_j(m),m),
\end{equation}
Note that $\ucb_j(m)$ and $\lcb_j(m)$ are well defined for each action $j\in\cK$ at all rounds $m$ since we do not assume that the action should be in the current active set to define these. Hence, if an action got eliminated in round $m'<m$, $T_j(m)=T_j(m')\: \text{and}\: \hat{f}_j(m) = \hat{f}_j(m') \: \forall m\ge m'$ i.e. the values of $T_j(m)$ and $\hat{f}_j(m)$ are frozen after the action got eliminated.

The proof technique is inspired by the that of the UCB-LP policy from \cite{buccapatnam2018reward}. We will analyze the regret by conditioning over two disjoint events defined as the \textit{good-event} $E$ and the \textit{bad-event} $E^c$. 
In order to define these events, we assume that we run a \textit{special version} of the the UCB-LP-A policy that does not simply terminate at time step $T$ but rather terminates at a time $t = \max(T_{new}, T)$, where $T_{new}$ is the time when it has completed the round $m^{\max}_a := \max_j m_{j,a}$ for activation set $\cK_a, \:\forall a \in [A]$. 
This ensures that we have atleast completed the round $m^{\max}_a$ for each activation set $\cK_a$, since this is crucial for the definition of the events $E$ and $E^c$.

\subsection{Definition of events $E$ and $E^c$}\label{sec:define_E}
Let $E$ be defined as the \textit{good event} where, for each activation set $\cK_a$, the optimal action $i^*_a$ of this set eliminates each suboptimal action $j\in \cK_a$ in or before its respective round $m_{j,a}$.
We assume that this event occurs with high probability.
Mathematically, it can be written down as an intersection of two events $E_1$ and $E_2$, which are defined as:
\begin{align}
    E_1 =& \Bigl\{\forall a \in [A], \text{optimal action}\: i^*_a\: \text{was not eliminated}\: \forall\: m: 0 \le m \le (m^{\max}_a-1)\Bigl\}\nonumber\\
    =& \bigcap_{a=1}^A \bigcap_{m=0}^{m^{\max}_a - 1} \bigcap_{j\in \cU_a} \{ \text{optimal action $i^*_a$ is not eliminated by action $j$ in round $m$}\}\\
    E_2 =& \Bigl\{ \forall a \in [A], \forall j \in \cU_a: \{\exists m\in\{0,...,m_{j,a}\}\: \ucb_j(m) < \lcb_{i^*_a}(m) \}\: \cup \nonumber\\
    &\quad\quad\quad\quad\{\text{optimal action $i^*_a$ is eliminated in or before round $(m_{j,a}-1)$}\} \Bigl\} \nonumber\\
    =& \bigcap_{a=1}^A \bigcap_{j\in \cU_a} \Bigl\{ \Bigl\{\bigcup_{m=0}^{m_{j,a}} \left\{ \ucb_j(m) < \lcb_{i^*_a}(m) \right\}\Bigl\}\: \cup \\
    &\quad\quad\quad\quad\{\text{optimal action $i^*_a$ is eliminated in or before round $(m_{j,a}-1)$}\} \Bigl\}
\end{align}
Event $E$ is just defined as: $E = E_1 \cap E_2$. Now, to define the \textit{bad event} $E^c$, we write the mathematical expressions for $E^c_1$ and $E^c_2$, which are complements of the events $E_1$ and $E_2$, respectively, given as:
\begin{align}
    E^c_1 =& \bigcup_{a=1}^A \bigcup_{m=0}^{m^{\max}_a - 1} \bigcup_{j\in \cU_a} \{ \text{optimal action $i^*_a$ is eliminated by action $j$ in round $m$}\}\\
    E^c_2 =& \bigcup_{a=1}^A \bigcup_{j\in \cU_a} \Bigl\{\Bigl\{ \bigcap_{m=0}^{m_{j,a}} \left\{ \ucb_j(m) \ge \lcb_{i^*_a}(m) \right\}\Bigl\}\: \cap \nonumber\\
    &\quad\quad\quad\quad\{\text{optimal action $i^*_a$ is not eliminated in or before round $(m_{j,a}-1)$}\} \Bigl\}
\end{align}

Hence, the event $E^c = E^c_1 \cup E^c_2$. Now, since we assume event $E$ corresponds to the event that would occur with high probability, event $E^c$ corresponds to the event that occurs with low probability. 
Now, to find the upper bound for the probability of $E^c$, we need to find upper bounds for the probability of $E^c_1$ and $E^c_2$.
Calculating the bound for $\bP(E^c_1)$ - 
\begin{align}\label{eq:prob_E_c_1}
    \bP(E^c_1) =& \sum_{a=1}^A \sum_{m=0}^{m^{\max}_a-1} \sum_{j\in\cU_a}  \bP(\text{optimal action $i^*_a$ is eliminated by action $j$ in round $m$}) \nonumber\\
    \le& \sum_{a=1}^A \sum_{m=0}^{m^{\max}_a-1} \sum_{j\in\cU_a } \dfrac{2}{T \tdel{m}^2}\quad \left( \text{using Lemma \ref{lemma:p_m_j}} \right) \nonumber\\
    \le& \sum_{a=1}^A \left(\dfrac{2|\cU_a|}{T}\right) \sum_{m=0}^{m^{\max}_a-1} \dfrac{1}{\tdel{m}^2} \nonumber\\
    \le& \sum_{a=1}^A \left(\dfrac{2|\cU_a|}{3T}\right) \dfrac{1}{\tdel{m^{\max}_a}^2} \left(\text{using $\tdel{m}=2^{-m}$ and} \sum_{m=0}^{m^{\max}_a-1} 4^m < \dfrac{4^{m^{\max}_a}}{3}\right) \nonumber\\
    \le& \dfrac{32}{3T} \sum_{a=1}^A \dfrac{|\cU_a|}{(\del{a}^{\min})^2} \left(\Delta_a^{\min}= \min_j \del{j,a}\: \text{and using \ref{eq:delta_conditions} for } j=\arg \max_i m_{j,a} = \arg \min_i \del{j,a} \right)
\end{align}
Now, let's calculate the upper bound for $\bP(E^c_2)$ - 
\begin{align}\label{eq:prob_E_c_2}
    \bP(E^c_2) =& \sum_{a=1}^A \sum_{j\in \cU_a} \bP \Bigl\{\Bigl\{ \bigcap_{m=0}^{m_{j,a}} \left\{ \ucb_j(m) \ge \lcb_{i^*_a}(m) \right\}\Bigl\} \:\cap \nonumber\\
    &\quad\quad\quad\quad\{\text{optimal action $i^*_a$ is not eliminated in or before round $(m_{j,a}-1)$}\} \Bigl\}  \nonumber\\
    \le& \sum_{a=1}^A \sum_{j\in \cU_a} \bP \Bigl\{ \left\{ \ucb_j(m_{j,a}) \ge \lcb_{i^*_a}(m_{j,a}) \right\}\: \cap \nonumber\\
    &\quad\quad\quad\quad\{\text{optimal action $i^*_a$ is not eliminated in or before round $(m_{j,a}-1)$}\} \Bigl\}  \nonumber\\
    \le& \sum_{a=1}^A \sum_{j\in \cU_a} \bP \left\{\text{optimal action $i^*_a\in \cB_{m_{(j,a)},a}$ and it did not eliminate action $j$ in round $m_{j,a}$} \right\} \nonumber\\
    \le& \sum_{a=1}^A \sum_{j\in \cU_a} \dfrac{2}{T\tdel{m_{j,a}}^2} \left(\text{using Lemma \ref{lemma:j_not_elim_m_j}}\right) \nonumber\\
    \le& \dfrac{32}{T} \sum_{a=1}^A \sum_{j\in \cU_a} \dfrac{1}{\del{m_{j,a}}^2} \left(\text{using $\dfrac{1}{\tdel{m_{(j,a)}}} \le \dfrac{4}{\del{j,a}}$ from \ref{eq:delta_conditions}}\right)
\end{align}

\subsection{Regret Analysis}
Recall that $R(T)$ represents the regret till the given time horizon $T$. Also recall that to define events $E$ and $E^c$, we assume that we run a \textit{special version} of the Algorithm \ref{alg:ucb-lp-new}, where we terminate the algorithm at time $t=\max(T_{new}, T)$ where $T_{new}$ is the time when it has completed the round $\max_j m_{j,a}$ for activation set $\cK_a, \:\forall a \in [A]$. In case $T<T_{new}$, i.e., the time horizon is less than the runtime of this special run, then the regret $R(T)$ would just be the accumulated regret counted only till time steps $T$ and not $T_{new}$. This makes sure that $R(T)$ is still referring to the same quantity even after introducing the events defined under \textit{special version} of the algorithm.

Now, we can calculate $\bE[R(T)]$ by conditioning over events $E$ and $E^c$ - 
\begin{align}
    \bE[R(T)] = \bE[R(T)|E]\bP(E) + \bE[R(T)|E^c]\bP(E^c)
\end{align}

\subsubsection{Calculating upper bound for $\bE[R(T)|E]$}
We would assume the event $E$ is given to be true throughout this part of the analysis.
Note that in Algorithm \ref{alg:ucb-lp-new}, we have two phases called the \textit{forced-sync phase} and \textit{independent phase}. The forced-sync phase corresponds to the phase of the algorithm when the round counters $m_a$ for each set $\cK_a$ are synced together, and we use the $z^*_{j,a}$ values to sample actions efficiently using the network side-information structure. The independent phase corresponds to the phase when the algorithm runs basic UCB with elimination as given in \cite{auer2010ucb} for each activation set independently, and we don't use $z^*_{j,a}$ based sampling. 

Define $\bar{m} := \max \Bigl\{ m \in \cM: \dfrac{1}{v_{\min}} \le 2 \tdel{m} \sum_{j \in \bigcup_a\cG_{m,a}} R_j(\bar{\cG}_m)\Bigl\}$, where $R_i(\bar{\cG}_m)$ can be calculated using Equation \ref{eq:cal_r} for action $i$, where $\bar{\cG}_m=\{\cG_{m,1},...,\cG_{m,A}\}$ with $\cG_{m,a} = \{j\in\cK_a:m_{j,a}\ge m\}$.
$\bar{\cG}_m$ is basically a special value of $\bar{\cB}_m$, which is the set containing the set of active actions at the start of round $m$ for all $\cK_a$.

Given event $E$, this indicates the last round $m$ for which we operate under the forced-sync phase.
Let $\cD_a := \{ j\in\cU_a: m_{j,a} > \bar{m} \}$ be the set of all suboptimal actions $j\in\cU_a$ that will be eliminated after round $\bar{m}$ given event $E$.
This means that the actions $j\in\cD_a$ experience both the phases, whereas actions $j\in\cU_a \setminus \cD_a$ only experience the forced-sync phase before getting eliminated.

Let $R^\sync(T)$ and $R^\ind(T)$ be the contributions from the forced-sync phase and the independent phase of the algorithm to make up $R(T)$, i.e., $\bE[R(T)] = \bE[R^\sync(T)] + \bE[R^\ind(T)]$.
Now, since we are interested in calculating an upper bound for $\bE[R(T)|E]$, we would need to compute upper bound to $\bE[R^\sync(T)|E]$ and $\bE[R^\ind(T)|E]$.



\paragraph{Calculating upper bound for $\bE[R^\sync(T)|E]$:} 
Let $N^\sync_{j,a}(m)$ denote the number of times action $j\in\cK_a$ is pulled during the forced-sync phase in round $m$.
Using Equation \ref{eq:regret_def}, but using rounds as base unit instead of time, we can write the conditional expectation $\bE[R^\sync(T)|E]$ as - 
\begin{equation}\label{eq:r_sync_def}
    \bE[R^\sync(T)|E] = \sum_{m=0}^{\bar{m}} \sum_{a=1}^A \sum_{j\in\cU_a} \bE[N^\sync_{j,a}(m)] \del{j,a}
\end{equation}
We need to get an upper bound on expected value $\mathbb{E}[N_{j,a}^{\text{sync}}(m)]$.
Lets just focus on round $m$ in the forced-sync phase.
Recall that in UCB-LP-A policy in Algorithm \ref{alg:ucb-lp-new}, we need at least $n(m)$ samples for each active action for each activation set $\cK_a$, before we go to the elimination step for that set $\cK_a$ at the end of round $m$. 
In the forced-sync phase, local round counters are synced and each activation set $\cK_a$ moves to the elimination step at the same time, i.e., as soon as each active action $j\in\cK_a$ for each $\cK_a,a\in[A]$, gets to $n(m)$ samples till the current time step during round $m$.

In this phase, when set $\cK_a$ is active, we sample each action (not just each active action) $i\in\cK_a$ with probability $\dfrac{z_{i,a}^*}{Z_a^*}$, where $Z_a^* = \sum_{j \in \cK_a} z_{j,a}^*$.
Recall that this phase is chosen with the idea that, for round $m$, getting to $n(m)$ samples for each action for each activation set $\cK_a$ using the $z^*$-sampling is much efficient than only sampling each \textit{active action} for each activation set $\cK_a$ to get to $n(m)$ samples (uniform sampling).
Now, recall that actions are made of subsets of base-arms, and we assumed that there are no useless base-arms. 
Hence, for \textit{each action} to get to $n(m)$ samples, we need \textit{each base-arm} to get to $n(m)$ samples. 
Let $v_i$ be the global rate at which base-arm $i\in\cN$ accumulates observations from different activation sets $\cK_a$. This is given by:
\begin{equation}\label{eq:v_i_def}
    v_i = \sum_{a=1}^A p_a \left( \sum_{j \in \cK_a \cap \cP_i} \frac{z_{j,a}^*}{Z_a^*} \right)
\end{equation}
The length of the round is determined by the ``slowest" base-arm (one with the minimum accumulation rate $v_i$).
At the start of round $m$ in forced-sync phase, each action already has at least $n(m-1)$ samples, which implies that each base-arm also already has at least $n(m-1)$ samples. 
In expectation, if the slowest base-arm gets $n$ samples in the current round, then each base-arm gets at least $n$ samples in this round.
Since, we don't know exactly which actions have more than $n(m-1)$ samples, we can get an upper bound for $\E[N^\sync_{j,a}(m)]$ by considering the worst case where the ``slowest" base-arm needs $\Delta n_m = [n(m) - n(m-1)]$ samples in round $m$ to end round $m$.

Let $T_m$ be the total number of time steps (total pulls) in this phase. The expected length is inversely proportional to the minimum rate:
\begin{equation}\label{eq:exp_t_m_def}
    \mathbb{E}[T_m] = \frac{\Delta n_m}{\min_{i \in \cN} v_i}
\end{equation}
From the coverage constraints of the LP minimization problem $P_1$ in \ref{eq:lp_problem_def}, we know that $\sum_{a} p_a \sum_{j \in \cK_a \cap \cP_i} z_{j,a}^* \ge 1, \forall i\in\cN$. Using this, we can lower bound $v_i$ as follows:
\begin{align}
    v_i &= \sum_{a=1}^A p_a \sum_{j \in \cK_a \cap \cP_i} \frac{z_{j,a}^*}{Z_a^*} \nonumber\\
    &\ge \frac{1}{\max_a Z_a^*} \underbrace{\left( \sum_{a=1}^A p_a \sum_{j \in \cK_a \cap \cP_i} z_{j,a}^* \right)}_{\ge 1 \text{ (LP Constraint)}} \nonumber\\
    &\ge \frac{1}{Z_{\max}^*} \left(\text{where}\: Z^*_{\max} = \max_a Z^*_a\right)
\end{align}
Substituting this back into the expectation for $T_m$ in Equation \ref{eq:exp_t_m_def}:
\begin{equation}\label{eq:exp_t_m_bound}
    \mathbb{E}[T_m] \le \frac{\Delta n_m}{1 / Z_{\max}^*} = \Delta n_m Z_{\max}^*
\end{equation}
Now, we calculate the expected number of times set $\cK_a$ appears in round $m$, denoted $E[N_a(m)]$:
\begin{equation}\label{eq:exp_N_a_def}
    \E[N_a(m)] = p_a \E[T_m] \le p_a (\Delta n_m Z_{\max}^*)
\end{equation}
Finally, the expected number of pulls for action $j\in\cK_a$ in round $m$ of the forced-sync phase is the expected number of times set $\cK_a$ appears in round $m$ multiplied by the conditional probability of choosing action $j$ given set $\cK_a$ is active:
\begin{align}\label{eq:exp_N_j,a_bound}
    \mathbb{E}[N^{\text{sync}}_{j,a}(m)] &= \mathbb{E}[N_a(m)] \cdot \left( \frac{z_{j,a}^*}{Z_a^*} \right) \nonumber\\
    &\le \left( p_a \Delta n_m Z_{\max}^* \right) \left( \frac{z_{j,a}^*}{Z_a^*} \right) = (p_a z_{j,a}^*) \left( \frac{Z_{\max}^*}{Z_a^*} \right) \Delta n_m
\end{align}
Using this in Equation \ref{eq:r_sync_def}, we get - 
\begin{align}\label{eq:r_sync_bound}
    \bE[R^\sync(T)|E] \le& \sum_{m=0}^{\bar{m}} \sum_{a=1}^A \sum_{j\in\cU_a} \left\{(p_a z_{j,a}^*) \left( \frac{Z_{\max}^*}{Z_a^*} \right) \Delta n_m\right\} \del{j,a}\nonumber\\
    =& \sum_{a=1}^A \sum_{j\in\cU_a} \left\{(p_a z_{j,a}^*) \left( \frac{Z_{\max}^*}{Z_a^*} \right) \sum_{m=0}^{\bar{m}} \Delta n_m\right\} \del{j,a} \nonumber\\
    =& \sum_{a=1}^A \sum_{j\in\cU_a} \left\{(p_a z_{j,a}^*) \left( \frac{Z_{\max}^*}{Z_a^*} \right) n(\bar{m})\right\} \del{j,a} 
\end{align}

\paragraph{Calculating upper bound for $\E[R^\ind(T)|E]$}
Let $N^\ind_{j,a}(m)$ denote the number of times action $j\in\cK_a$ is pulled during the independent phase in round $m$.
Unlike forced-sync phase, we do not need to pull each action from each activation set in the independent phase, and rather we only pull active actions. Given event $E$, we know that only actions $j\in\cD_a$ for each activation set $\cK_a$, will contribute to regret for the independent phase.
Also, given event $E$, we know that each suboptimal action $j\in\cU_a$ for each activation set $\cK_a$ will get eliminated in or before round $m_{j,a}$.

Now, to ensure that we get to at least $n(m)$ samples for each active action for a given set $\cK_a$, at the end of round $m$ for that set, we need to pull each active action a maximum of $\Delta n_m = [n(m) - n(m-1)]$ times.
Hence, $\E[N^\ind_{j,a}(m)]\le \Delta n_m, \forall m>\bar{m}$.
Using this, we can write the conditional expectation $\bE[R^\ind(T)|E]$ as - 
\begin{align}\label{eq:r_ind_bound}
    \bE[R^\ind(T)|E] =& \sum_{a=1}^A \sum_{j\in\cD_a} \sum_{m=\bar{m}}^{m_{j,a}} \bE[N^\ind_{j,a}(m)] \del{j,a} \nonumber\\
    \le& \sum_{a=1}^A \sum_{j\in\cD_a} \sum_{m=\bar{m}}^{m_{j,a}} \Delta n_m \del{j,a} \nonumber\\
    =& \sum_{a=1}^A \sum_{j\in\cD_a} \left\{n(m_{j,a}) - n(\bar{m})\right\} \del{j,a}
\end{align}

\paragraph{Upper Bound for $\E[R(T)|E]$}
Now, using inequalities in \ref{eq:r_sync_bound} and \ref{eq:r_ind_bound}, we can write regret contributions of every suboptimal action $j\in\cU_a$ for every $\cK_a$ as a single term, i.e., combining contributions appearing in both $\E[R^\sync(T)|E]$ and $\E[R^\ind(T)|E]$. Let $\gamma_{j,a}:=(p_a z_{j,a}^*) \left(\dfrac{Z_{\max}^*}{Z_a^*}\right)$. Then, we can write the upper bound for $\E[R(T)|E]$ as  - 
\begin{align}\label{eq:r_t_given_E_bound}
    \E[R(T)|E] &\le \sum_{a=1}^A \sum_{j \in \cU_a \setminus \cD_a} \gamma_{j,a} n(\bar{m}) \Delta_{j,a} + \sum_{a=1}^A \sum_{j \in \cD_a} \left[\gamma_{j,a}n(\bar{m}) -n(\bar{m}) + n(m_{j,a})\right]  \Delta_{j,a} \nonumber\\
    &\le \sum_{a=1}^A \sum_{j \in \cU_a \setminus \cD_a} \gamma_{j,a} \left( \frac{2 \log(T \tilde{\Delta}_{\bar{m}}^2)}{\tilde{\Delta}_{\bar{m}}^2} + 1\right) \Delta_{j,a} \quad+\nonumber\\
    &\quad\sum_{a=1}^A \sum_{j \in \cD_a} \left[\gamma_{j,a}\left( \frac{2 \log(T \tilde{\Delta}_{\bar{m}}^2)}{\tilde{\Delta}_{\bar{m}}^2}+1\right) - \left( \frac{2 \log(T \tilde{\Delta}_{\bar{m}}^2)}{\tilde{\Delta}_{\bar{m}}^2}\right) + \left( \frac{2 \log(T \tilde{\Delta}_{m_{(j,a)}}^2)}{\tilde{\Delta}_{m_{(j,a)}}^2} +1\right) \right] \Delta_{j,a}\nonumber\\
    &\le \sum_{a=1}^A \sum_{j \in \cU_a \setminus \cD_a} \gamma_{j,a} \left( \frac{2 \log(T \tilde{\Delta}_{\bar{m}}^2)}{\tilde{\Delta}_{\bar{m}}^2} \right) \Delta_{j,a} +\sum_{a=1}^A \sum_{j \in \cD_a} \left[(\gamma_{j,a}-1)\left( \frac{2 \log(T \tilde{\Delta}_{\bar{m}}^2)}{\tilde{\Delta}_{\bar{m}}^2}\right) + \left( \frac{2 \log(T \tilde{\Delta}_{m_{(j,a)}}^2)}{\tilde{\Delta}_{m_{(j,a)}}^2}\right) \right] \Delta_{j,a} \nonumber\\
    &\quad+\sum_{a=1}^A \sum_{j \in \cU_a} (\gamma_{j,a} + 1) \Delta_{j,a}
\end{align}

\subsubsection{Calculating upper bound for $\bE[R(T)|E^c]$}
Using inequalities \ref{eq:prob_E_c_1} and \ref{eq:prob_E_c_2}, we can see that the probability of the event $E^c$, i.e., $\bP[E^c_1] + \bP[E^c_2]$, is upper bounded by an expression inversely dependent on $T$. 
This means that it is a low probability event with the probability going to $0$ as $T$ goes to $\infty$.
Hence, given event $E^c$, we can simply bound the regret assuming that we pulled the worst possible action, i.e., the action with $\max_{j,a} \del{j,a}$, every time till time $T$. This gives us -
\begin{align}\label{eq:r_t_given_Ec_bound}
    \E[R(T)|E^c] \le \max_{j,a} \Delta_{j,a}
\end{align}

\subsubsection{Calculating upper bound for $\bE[R(T)]$}
Finally, using the upper bounds for $\E[R(T)|E]$ and $\E[R(T)|E^c]$ from inequalities \ref{eq:r_t_given_E_bound} and \ref{eq:r_t_given_Ec_bound}, using inequalities \ref{eq:prob_E_c_1} and \ref{eq:prob_E_c_2} to bound $\bP(E^c)$, using $\bP(E)\le1$ and combining terms that don't depend on time $T$ in $O(K)$, we get to our desired result.
\end{proof}

\section{Proof of Theorem \ref{thm:baseline_result}}\label{app:proof_thm_2}
\begin{proof}
The proof for the baseline, where we perform UCB with elimination, as given in \cite{auer2010ucb} and what we called UCB-E for the simulations, to each activation set $\cK_a$ independently, follows the exact logic like the proof for Theorem \ref{thm:new_ucb_lp_result} given in Appendix \ref{app:proof_thm_1}.
However, it differs only in the fact that the baseline does not use $z^*$-sampling to get samples for active actions.
This means that we perform UCB with elimination for the current active set $\cK_a$ and pull one action at the current time $t$ and we continue pulling actions for this round whenever set $\cK_a$ appears next time.
This is basically dealing with different sets $\cK_a$ independently and continuing where we left for each set whenever that set appears again.

We assume the same good event $E$ and it's compliment $E^c$ which means that the probability of the good event and the bad event uses the same upper bounds as in the proof in Appendix \ref{app:proof_thm_1}.
Since, under the good event, we know that each suboptimal action $j\in\cU_a$, for all $a\in[A]$, will be eliminated by their respective rounds $m_{j,a}$, we can bound the number of times each action will be pulled by $n(m_{j,a}) = \left\lceil \dfrac{2 \log(T \tilde{\Delta}_{m_{(j,a)}}^2)}{\tilde{\Delta}_{m_{(j,a)}}^2} \right\rceil$.

This gives us the upper bound on the regret contribution given good event $E$, i.e., $\bE[R(T)|E]$. 
Similar to the proof in Appendix \ref{app:proof_thm_1}, we can bound the regret contribution given the bad event $E^c$, i.e., $\bE[R(T)|E]$ assuming that we pull the worst action possible, i.e., the action with $\max_{j,a} \del{j,a}$, every time till time $T$.
Now using both these regret contributions, using the bounds for the probabilities of the good and bad events, and combining the time independent terms into $O(K)$, we get to out desired result.
\end{proof}

\newpage
\section{Supplementary Material}


\begin{lemma}[Chernoff-Hoeffding Inequality]\label{lemma:chernoff}
Let $X_{1},...,X_{n}$ be a sequence of random variables with support [0,1] and $\mathbb{E}[X_{t}]=\mu$ for all $t\le n$. Let $S_{n}=\frac{1}{n}\sum_{j=1}^{n}X_{j}$ Then, for all $\epsilon>0,$ we have,
$$\bP[S_{n}\ge\mu+\epsilon]\le e^{-2n\epsilon^{2}}$$
$$\bP[S_{n}\le\mu-\epsilon]\le e^{-2n\epsilon^{2}}$$
\end{lemma}

\begin{lemma}\label{lemma:j_not_elim_m_j}
The probability that suboptimal action $j\in\cK_a$ is not eliminated by its optimal action $i^*_a$ for set $\cK_a$ in round $m_{j,a}$, is at most $\dfrac{2}{T \tdel{m_{(j,a)}}^{2}}.$
\end{lemma}
\begin{proof}
Let $\hat{f}_{j}(m)$ and $\hat{f}_{*}(m)$ be the sample mean of all observations for action $j\in\cK_a$ and set $\cK_a$'s optimal action $i^*_a$ at the end of round $m$ respectively. 
Let $T_j(m)$ and $T_*(m)$ be the total observation count for suboptimal action $j\in\cK_a$ and optimal arm $i^*_a$ at the end of round $m$ respectively.
Also, let $\mu_j$ and $\mu_*$ represent the true means of the suboptimal action $j$ and optimal action $i^*_a$ respectively.
Let $\cB_{m,a}$ denote the set of active actions at the start of round $m$ for set $\cK_a$.
Using UCB-LP-A policy as in Algorithm \ref{alg:ucb-lp-new}, at the end of round $m$ for a given set $\cK_a$, we know that for all active actions $j\in\cB_{m,a}$, we have at least $n(m)=\left\lceil\dfrac{2\log(T\tilde{\Delta}_{m}^{2})}{\tilde{\Delta}_{m}^{2}}\right\rceil$ observations. 

Using the definition of $n(m)$ and $m_{j,a}$ and using $W(k,m) = \sqrt{\dfrac{\log(T \tilde{\Delta}^2_m)}{2k}}$, in round $m_{j,a}$ we have - 
\begin{equation}\label{eq:not_elim_j_k_a_1}
    W(n(m_{j,a}),m_{j,a}) \le \frac{\tdel{m_{(j,a)}}}{2}<\frac{\del{j,a}}{4},
\end{equation}
    Now, for $m=m_{j,a}$, consider the following conditions -
\begin{align}
\label{eq:not_elim_j_k_a_2}
\left(\hat{f}_{j}(m_{j,a})\le\mu_{j}+ W(n(m_{j,a}),m_{j,a})\right) \text{ and } \left(\hat{f}_{*}(m_{j,a})\ge\mu_{*}-W(n(m_{j,a}),m_{j,a})\right),
\end{align}
In the elimination phase of the UCB-LP-A policy in round $m_{j,a}$, if the conditions in \ref{eq:not_elim_j_k_a_2} hold for action $j$ and $i^*_a$, then using property that $T_j(m_{j,a})\geq n(m_{j,a})$ and hence $W(n(m_{j,a}),m_{j,a})\ge W(T_j(m_{j,a}),m_{j,a}) \: \forall j \in\cB_{m,a}$, we have -
\begin{align}\label{eq:not_elim_j_k_a_3}
\hat{f}_{j}(m_{j,a})+W(T(m_{j,a}),m_{j,a}) &\le \hat{f}_{j}(m_{j,a}) + W(n(m_{j,a}),m_{j,a}) \nonumber\\ 
&\le\mu_{j} + 2W(n(m_{j,a}),m_{j,a}) \:\:(\text{using \ref{eq:not_elim_j_k_a_2}}) \nonumber\\
&<\mu_{j} + \del{j,a} - 2W(n(m_{j,a}),m_{j,a}) \:\:(\text{using \ref{eq:not_elim_j_k_a_1}}) \nonumber\\
&=\mu_{*} - 2W(n(m_{j,a}),m_{j,a}) \nonumber\\
&\le\hat{f}_{*}(m_{j,a}) - W(n(m_{j,a}),m_{j,a}) \nonumber\\
&\le \hat{f}_{*}(m_{j,a})-W(T(m_{j,a}),m_{j,a}),
\end{align}
which implies that action $j$ is eliminated by $i^*_a$ in round $m_{j,a}$ (Algorithm \ref{alg:action-elimination}). 
Hence, the probability that action $j$ is not eliminated by $i^*_a$ in round $m_{j,a}$ is upper bounded by the probability that either one of the inequalities in \ref{eq:not_elim_j_k_a_2} does not hold. 
We say upper bounded since one of the inequalities not holding does not guarantee that action $j$ is not eliminated, since the 2 inequalities together form a sufficient condition but not a necessary condition. 
Using Chernoff-Hoeffding bound (Lemma \ref{lemma:chernoff}), we can calculate the probability of each inequality in \ref{eq:not_elim_j_k_a_2} failing as follows:
\begin{align}
\bP\left[\hat{f}_{j}(m_{j,a})>\mu_{j}+W(n(m_{j,a}),m_{j,a})\right] \le \bP\left[\hat{f}_{j}(m_{j,a})>\mu_{j}+W(T(m_{j,a}),m_{j,a})\right] \le \frac{1}{T\tdel{m_{(j,a)}}^{2}} \label{eq:not_elim_j_k_a_4}\\
\bP\left[\hat{f}_{*}(m_{j,a})<\mu_{*}-W(n(m_{j,a}),m_{j,a})\right] \le \bP\left[\hat{f}_{*}(m_{j,a})<\mu_{*}-W(T(m_{j,a}),m_{j,a})\right] \le
\frac{1}{T\del{m_{(j,a)}}^{2}} \label{eq:not_elim_j_k_a_4}
\end{align}
Summing the above two inequalities gives us that the probability that the suboptimal action $j\in\cK_a$ is not eliminated by its optimal arm $i^*_a$ for $\cK_a$ in round $m_{j,a}$, is at most $\dfrac{2}{T\tdel{m_{(j,a)}}^2}$.
\end{proof}

\begin{lemma}\label{lemma:cond_elim_opt}
If the conditions in \ref{eq:w_conditions} hold, then suboptimal action $j\in\cK_a$ cannot eliminate optimal action for set $\cK_a$, denoted by $*$, in round $m$, given that $T_*(m) = k_*$ and $T_j(m) = k_j$, where $T_j(m)$ and $\hat{f}_j(m)$ denote the total observation count and sample mean for action $j\in\cK_a$ at the end of round $m$ respectively.
\begin{equation} \label{eq:w_conditions}
    \left(\hat{f}_*(m) \ge \mu_* - W(k_*,m)\right) \:\text{and}\: \left(\hat{f}_j(m) \le \mu_j + W(k_j,m)\right), \text{where} \:\: W(k,m) = \sqrt{\dfrac{\log(T \tilde{\Delta}^2_m)}{2k}}
\end{equation}
\end{lemma}

\begin{proof}
Assuming \ref{eq:w_conditions} is true and using the fact that $(\mu_* - \mu_j) \ge 0$, we get:
\begin{align}
\label{eq:opt_arm_not_elim_1}
    \hat{f}_*(m) - \hat{f}_j(m) &\ge (\mu_* - \mu_j) - (W(k_*,m) + W(k_j,m)) \ge -(W(k_*,m) + W(k_j,m))
\end{align}
The condition for suboptimal action $j\in\cK_a$ eliminating optimal action $*$, in round $m$, given $T_*(m) = k_*$ and $T_j(m) = k_j$, is given by:
\begin{align}
    \textit{(UCB score of *)} = \: &\hat{f}_*(m) + W(k_*,m) < \hat{f}_j(m) - W(k_j,m) = \textit{(LCB score of } j) \nonumber\\
    \implies &\hat{f}_*(m) - \hat{f}_j(m) < -(W(k_*,m) + W(k_j,m)) \label{eq:opt_arm_not_elim_2}
\end{align}
Since \ref{eq:opt_arm_not_elim_1} and \ref{eq:opt_arm_not_elim_2} are contradictory, it implies that given the conditions in \ref{eq:w_conditions}, suboptimal action $j\in\cK_a$ cannot eliminate optimal action $*$ of set $\cK_a$.
\end{proof}

\begin{lemma}\label{lemma:p_m_j}
The probability that the optimal action $*$ of a given activation set $\cK_a$ is eliminated by a specific suboptimal action $j\in\cK_a$ in round $m$, denoted by $p_{m}^j$, is upper-bounded by $\dfrac{2}{T \tdel{m}^2}$.
\end{lemma}

\begin{proof}
Let $E$ be the event that the optimal action $*$ of a given activation set $\cK_a$ is eliminated by the suboptimal action $j\in\cK_a$ in round $m$. This elimination occurs if the following inequality holds:
\begin{equation} \label{eq:p_m_j_1}
    \hat{f}_*(m) + W(T_*(m), m) \le \hat{f}_j(m) - W(T_j(m), m)
\end{equation}
where $W(k, m) = \sqrt{\dfrac{\log(T\tilde{\Delta}_m^2)}{2k}}$ is the confidence width, and $T_*(m), T_j(m)$ are random variables representing the number of samples collected for actions $*$ and $j$ up to round $m$.

We first analyze the probability of this event conditioned on fixed sample counts i.e. $\bP(E|\: T_*(m) = k_*, T_j(m) = k_j)$, where $k_*$ and $k_j$ are integers such that $k_*, k_j \ge n(m)$. 
We know that the elimination condition in \ref{eq:p_m_j_1} given $T_*(m)=k_*, T_j(m)=k_j$, i.e., $\hat{f}_*(m)$ and $\hat{f}_j(m)$ are mean estimates using $k_*$ and $k_j$ samples respectively, will not be satisfied if the two conditions in Lemma \ref{lemma:cond_elim_opt} are satisfied. 
Hence, in order for suboptimal action $j$ to eliminate optimal action $*$, we need that atleast one of the two conditions in Lemma \ref{lemma:cond_elim_opt} should not be satisfied i.e. either one of the following conditions hold:

\begin{equation} \label{eq:p_m_j_2}
    \left(\hat{f}_*(m) < \mu_* - W(k_*,m)\right) \:\text{or}\: \left(\hat{f}_j(m) > \mu_j + W(k_j,m)\right), 
\end{equation}
We can bound the probability of above conditions using Chernoff-Hoeffding's inequality (Lemma \ref{lemma:chernoff}) to get:
\begin{align}\label{eq:p_m_j_3}
    \bP\left(\hat{f}_*(m) < \mu_* - W(k_*, m) \mid T_*(m)=k_*\right) &\le \exp\left(-2 k_* W^2(k_*, m)\right) \nonumber\\
    &\le \exp\left(-2 k_* \frac{\log(T\tilde{\Delta}_m^2)}{2k_*}\right) \le \frac{1}{T\tilde{\Delta}_m^2}
\end{align}
and similarly, 
\begin{align}\label{eq:p_m_j_4}
    \bP\left(\hat{f}_j(m) > \mu_j + W(k_j, m) \mid T_j(m)=k_j\right) \le \frac{1}{T\tilde{\Delta}_m^2}
\end{align}

The probability of suboptimal action $j$ eliminating optimal action $*$ is upper-bounded by the sum of probabilities given in inequalities \ref{eq:p_m_j_3} and \ref{eq:p_m_j_4}. Hence, we get:
\begin{equation}\label{eq:p_m_j_5}
\bP(E \mid T_*(m)=k_*, T_j(m)=k_j) \le \frac{1}{T\tilde{\Delta}_m^2} + \frac{1}{T\tilde{\Delta}_m^2} = \frac{2}{T\tilde{\Delta}_m^2}
\end{equation}
Crucially, this bound is independent of the specific values of $k_*$ and $k_j$.
Finally, we obtain the unconditional probability $p_{m}^j$ by summing over all valid sample counts using the Law of Total Probability:
\begin{align*}
    p_{m}^j &= \sum_{k_*\ge n(m)} \sum_{k_j \ge n(m)} \bP(E \mid T_*(m)=k_*, T_j(m)=k_j) \cdot \bP(T_*(m)=k_*, T_j(m)=k_j) \\
    &\le \sum_{k_* \ge n(m)} \sum_{k_j \ge n(m)} \left( \frac{2}{T\tilde{\Delta}_m^2} \right) \cdot \bP(T_*(m)=k_*, T_j(m)=k_j) \\
    &= \frac{2}{T\tilde{\Delta}_m^2} \underbrace{\sum_{k_* \ge n(m)} \sum_{k_j \ge n(m)} \bP(T_*(m)=k_*, T_j(m)=k_j)}_{\le 1} \le \frac{2}{T\tilde{\Delta}_m^2}
\end{align*}
\end{proof}

\end{document}